\documentclass[]{interact}

\usepackage{epstopdf}

\usepackage{graphicx}

\usepackage{color,verbatim}
\usepackage{amsmath}
\usepackage[linesnumbered,ruled,vlined]{algorithm2e}
\SetKwRepeat{Do}{do}{while}
\usepackage{subfigure}

\theoremstyle{plain}
\newtheorem{theorem}{Theorem}[section]
\newtheorem{lemma}[theorem]{Lemma}
\newtheorem{corrolary}[theorem]{Corollary}

\theoremstyle{definition}

\theoremstyle{remark}
\newtheorem{remark}{Remark}

\newcommand{\rset}{\mathbf{R}}

\newcommand{\dist}{\text{dist}}

\newcommand{\dom}{\text{dom}}
\newcommand{\prox}{\text{prox}}
\newcommand{\T}{\mathcal{T}}

\newcommand{\K}{\mathcal{K}}

\newcommand{\N}{\mathcal{N}}
\newcommand{\BO}{\mathcal{O}}

\providecommand{\norm}[1]{\lVert#1\rVert}
\newcommand{\Max}[1]{\max\left\{ #1 \right\}}
\newcommand{\Min}[1]{\min \left\{ #1 \right\} }
\newcommand{\BigO}[1]{\BO\left( #1 \right)}

\newcommand{\be}{\begin{equation}}
	\newcommand{\ee}{\end{equation}}

\begin{document}
	
\articletype{ARTICLE}

\title{Complexity of Inexact Proximal Point Algorithm for minimizing convex functions with Holderian Growth}

\author{
	\name{Andrei P\u atra\c scu 	and 
		Paul Irofti  \textsuperscript{a} \thanks{Email: andrei.patrascu@fmi.unibuc.ro, paul@irofti.net} }
	\affil{\textsuperscript{a} Research Center for Logic, Optimization and Security (LOS), \\
		Department of Computer Science, Faculty of Mathematics and Computer Science, \\
		University of Bucharest, Academiei 14, Bucharest, Romania.}
}

\maketitle

		



%
%

\date{Received: date / Accepted: date}

\maketitle

\begin{abstract}
	
\noindent Several decades ago the Proximal Point	Algorithm (PPA) started to gain a long-lasting attraction for both abstract operator theory and numerical optimization communities. Even in modern applications, researchers still use proximal minimization theory to design scalable algorithms that overcome nonsmoothness. Remarkable works as \cite{Fer:91,Ber:82constrained,Ber:89parallel,Tom:11} established tight relations between the convergence behaviour of PPA and the regularity of the objective function. In this manuscript we derive nonasymptotic iteration complexity of exact and inexact PPA to minimize convex functions under $\gamma-$Holderian growth: $\BigO{\log(1/\epsilon)}$ (for $\gamma \in [1,2]$) and $\BigO{1/\epsilon^{\gamma - 2}}$ (for $\gamma > 2$). In particular, we recover well-known results on PPA:  finite convergence for sharp minima and linear convergence for quadratic growth, even under presence of deterministic noise. Moreover, when a simple Proximal Subgradient Method  is recurrently called as an inner routine for computing each IPPA iterate, novel computational complexity bounds are obtained for Restarting Inexact PPA. Our numerical tests show improvements over existing restarting versions of the Subgradient Method.  

\end{abstract}


\section{Introduction}\label{sec:intro}

The problem of interest of our paper formulates as the following convex nonsmooth minimization:
\begin{align}\label{problem_of_interest}
	F^* = \min_{x \in \rset^n} \;\{  F(x) := f(x) + \psi(x)  \}. 
\end{align}
Here we assume that $f: \rset^n \mapsto \rset^{}$ is convex and $\psi: \rset^n \mapsto (-\infty,\infty]$ is convex, lower semicontinuous and proximable. By proximable function we refer to those functions whose proximal mapping is computable in closed form or linear time. The above model finds plenty of applications from which we shortly mention compressed sensing \cite{BecTeb:09fista},  sparse risk minimization \cite{Lin:18,XiaLin:14}  and graph-regularized models \cite{YanEla:16}.
Dating back to '60s, the classical Subgradient Methods (SGM) \cite{Sho:62,Sho:64,Pol:67SM1,Pol:69SM2,Pol:87book} established $\BigO{1/\epsilon^2}$ iteration complexity for minimizing convex functions up to finding $\hat{x}$ such that $f(\hat{x}) - f^* \le \epsilon$. Despite the fact that this complexity order is unimprovable for the class of convex functions, particular growth or error bounds properties can be exploited to obtain better lower orders. Error bounds and regularity conditions have a long history in optimization, systems of inequalities or projection methods: \cite{Ant:94,BurFer:93,Fer:91,Pol:67SM1,Pol:69SM2,Pol:87book,BolNgu:17,Hu:16,Luo:93}. 
Particularly, in the seminal works \cite{Pol:78,Pol:87book} SGM is proved to converge linearly towards weakly sharp minima of $F$. The optimal solutions $X^*$ is a set of weak sharp minima (WSM) if there exists $\sigma_F > 0$ such that
\begin{align*}
WSM: \qquad	F(x) - F^* \ge \sigma_F \dist_{X^*}(x), \quad \forall x \in \dom {F}.
\end{align*}
Acceleration of other first-order algorithms has been proved under WSM in subsequent works as \cite{Ant:94,BurFer:93,Dav:18,Rou:20}. Besides acceleration, \cite[ Section 5.2.3]{Pol:87book} introduces the ``superstability" of sharp optimal solutions $X^*$: under small perturbations of the objective function $F$ a subset of the weak sharp minima $X^*$ remains optimal for the perturbed model. The superstability of WSM was used in \cite{Pol:78,Ned:10} to show the robustness of inexact SGM. In short, using low persistently perturbed subgradients at each iteration, the resulted perturbed SGM still converges linearly to $X^*$. In the line of these results, we also show in our manuscript that similar robustness holds for the proximal point methods under WSM.

Other recent works as \cite{Yan:18,Joh:20,Nec:19,Lu:20,Kor:76,Tom:11,Li:12,Hu:16,Fre:18,Luo:93,Gil:12,Ren:14,BolNgu:17,JudNes:14uniform} look at a suite of different growth regimes besides WSM and use them to improve the complexity of first-order algorithms. Particularly, in our paper we are interested in the $\gamma-$Holderian growth : let $\gamma \ge 1$
\begin{align*}
\gamma-HG: \qquad	F(x) - F^* \ge \sigma_F \dist_{X^*}^{\gamma}(x), \quad \forall x \in \dom {F}.
\end{align*}
Note that $\gamma-HG$ is equivalent to the Kurdyka–Łojaziewicz (KL)
inequality for convex, closed, and proper functions, as shown in \cite{BolNgu:17}.
It includes the class of uniformly convex functions analyzed in \cite{JudNes:14uniform} and,
obviously, it covers the sharp minima WSM, for $\gamma = 1$. 
The Quadratic Growth (QG), covered by $\gamma = 2$, was analyzed in a large suite of previous works \cite{Lu:20,Yan:18,Luo:93,Nec:19} and, although is weaker than strong convexity, it could be essentially exploited (besides Lipschitz gradient continuity) to show $\BigO{\log(1/\epsilon)}$ complexity of proximal gradient methods11. Our analysis recover similar complexity orders under the same particular assumptions.

\noindent Some recent works \cite{Yan:18,Joh:20,Gil:12,Fre:18,Ren:14} developed restarted SGM schemes, for minimizing convex functions under $\gamma$-HG or WSM, and analyzed their theoretical convergence and their natural dependence on the growth moduli $\gamma$ and $\sigma_F$. Restarted SubGradient (RSG) of \cite{Yan:18} and Decaying Stepsize - SubGradient (DS-SG) of \cite{Joh:20} present iteration complexity estimates of $\BigO{\log(1/\epsilon)}$ under WSM and $\BigO{\frac{1}{\epsilon^{2(\gamma-1)}}}$ bound under $\gamma-$(HG) in order to attain $\dist_{X^*}(x) \le \epsilon$. These bounds are optimal for bounded gradients functions, as observed by \cite{NemNes:85}. Most SGM schemes are dependent up to various degrees on the knowledge of problem information. For instance, RSG and DS-SG rely on lower bounds of optimal value $F^*$ and knowledge of parameters $\sigma_F, \gamma$ or other Lipschitz constants. The restartation is introduced in order to avoid the exact estimation of modulus $\sigma_F$. Also, our schemes allows estimations of problem moduli such as $\gamma, \sigma_F$, covering cases when these are not known. In the best case, when estimations are close to the true parameters, similar complexity estimates $\BigO{\frac{1}{\epsilon^{2(\gamma-1)}}}$ are provided in terms of subgradient evaluations. Moreover, by exploiting additional smooth structure we further obtain lower estimates.

The work of \cite{JudNes:14uniform} approach the constrained model, i.e. $\psi$ is the indicator function of a closed convex set, and assume $\gamma-$uniform convexity:
\begin{align*}
	f(\alpha x + (1-\alpha) y) \le \alpha f(x) +& (1-\alpha)f(y) \\
	 & - \frac{1}{2}\sigma_f \alpha(1-\alpha)[\alpha^{\gamma - 1} + (1-\alpha)^{\gamma - 1}]\norm{x-y}^{\gamma},
\end{align*}
for all feasible $x,y$ and $\gamma \ge 2$. The authors obtain optimal complexity bounds 
$\BigO{\sigma_f^{-2/\gamma} \epsilon^{-2(\gamma - 1)}}$ when the subgradients of $f$ are bounded. Moreover, their restartation technique are adaptive to growth modulus $\gamma$ and parameter $\sigma_F$, up to a fixed number of iterations.

Inherent for all SGMs, the complexity results of these works essentially requires the boundedness of the subgradients, which is often natural for nondifferentiable functions. However, plenty of convex objective functions coming from risk minimization, sparse regression or machine learning presents, besides their particular growth, a certain smoothness degree which is not compatible with the subgradient boundedness assumption. Enclosing the feasible domain in a ball is an artificial remedy used to further keep the subgradients bounded, which however might load the implementation with additional uncertain tuning heuristics. Our analysis shows how to exploit smoothness in order to improve the complexity estimates.

The analysis of \cite{Rou:20} investigates the effect of restarting over the optimal first-order schemes under $\gamma$-HG and $\nu$-Holder smoothness, starting from results of \cite{NemNes:85}. For $\psi =0$, $\epsilon-$suboptimality is reached after $\BigO{\log(1/\epsilon)}$ accelerated gradient iterations if $\nabla F$ is Lipschitz continuous and $2-$Holder growth holds, or after $\BigO{ 1/\epsilon^{\frac{\gamma-2}{2}}   }$ iterations when the growth modulus is larger than $2$. In general, if $\nabla F$ is $\nu-$Holder continuous, they restart the Universal Gradient Method and obtain an overall complexity of $\BigO{\log(1/\epsilon)}$ if $\gamma = \nu$, or $\BigO{ 1/\epsilon^{\frac{2(\gamma-\nu-1)}{2\nu-1}}   }$ if $\gamma > \nu$. Although these estimates are unimprovable and better than ours, in general the  implementation of the optimal schemes requires complete knowledge of growth and smoothness parameters. 

\vspace{5pt}

\noindent  Several decades ago the \textit{Proximal Point	Algorithm (PPA)} started to gain much attraction for both abstract operator theory and the numerical optimization communities. Even in modern applications, where large-scale nonsmooth optimization arises recurrently, practitioners still inspire from proximal minimization theory to design scalable algorithmic techniques that overcomes nonsmoothness. The powerful PPA  iteration consists mainly in the recursive evaluation of the proximal operator associated to the objective function.  The proximal mapping is based on the infimal convolution with a metric function, often chosen to be the squared Euclidean norm:
$ \prox_{\mu}^F(x) := \arg\min_z F(z) + \frac{1}{2\mu}\norm{z - x}^2.$
The Proximal Point recursion: 
\begin{align*}
x^{k+1} = \prox_{\mu}^F(x^k).
\end{align*}
became famous in optimization community when \cite{Roc:76,Roc:76augmented} and \cite{Ber:82constrained,Ber:89parallel} revealed its connection to various multipliers methods for constrained minimization, see also \cite{Nes:21,Nes:21b,Nem:04,Gul:91,Gul:92}. There are remarkable works that shown how the growth regularity is a key factor in the iteration complexity of PPA. 

Finite convergence of the exact PPA under WSM is proved by \cite{BurFer:93,Fer:91,Ant:94}.
Furthermore, in \cite{Ber:89parallel,Kor:76} can be found an extensive convergence analysis of the exact PPA and the Augmented Lagrangian algorithm under $\gamma-$(HG). Although the results and analysis are of a remarkable generality, they are of asymptotic nature (see \cite{Tom:11}). A nonasymptotic analysis is found in \cite{Tom:11}, where the equivalence between a Dual Augmented Lagrangian algorithm and a variable stepsize PPA is established. The authors analyze sparse learning models of the form: $\min_{x \in \rset^n} \; f(Ax) + \psi(x),$
where $f$ is twice differentiable with Lipschitz continuous gradient, $A$ a linear operator and $\psi$ a convex nonsmooth regularizer. Under $\gamma-$Holderian growth, ranging with $\gamma \in [1,2]$, they show nonasymptotic superlinear convergence rate of the exact PPA with exponentially increasing stepsize. 
For the inexact variant they kept further a slightly weaker superlinear convergence. 
The progress, from the asymptotic analysis of \cite{Roc:76,Kor:76} to a nonasymptotic one, is remarkable due to the simplicity of the arguments. However, a convergence rate of  inexact PPA (IPPA) could become irrelevant without quantifying the local computational effort spent to compute each iteration, since one inexact iteration of PPA requires the approximate solution the regularized  optimization problem. Among the remarkable references on inexact versions of various proximal algorithms are \cite{Mon:13,Sol:00,Sol:01,Sol:01b,Sol:99,Sol:99b,Lin:15,Lin:18,Mai:19,Shulgin:21,Nes:21,Nes:21b}.

We mention that, a small portion of the results on the WSM case, contained into this manuscript, has been recently published by the authors in \cite{PatIroLett:22}. However, we included it in the present manuscript for the sake of completeness.

\vspace{10pt}

\noindent \textbf{Contributions}. We list further our main contributions:

\noindent \textit{Inexact PPA under $\gamma-$(HG)}. We provide nonasymptotic iteration complexity bounds for IPPA to solve \eqref{problem_of_interest} under $\gamma-$HG, when $\gamma \ge 1$. In particular, we obtain $\BigO{\log(1/\epsilon)}$ for $\gamma \in [1,2]$ and, in the case of best parameter choice, $\BigO{1/\epsilon^{\gamma - 2}}$ for $\gamma > 2$, to attain $\epsilon$ distance to the optimal set. All these bounds require only convexity of the objective function $F$ and they are independent on any bounded gradients or smoothness. We could not find these nonasymptotic estimates in the literature for general $\gamma \ge 1$.

\vspace{5pt}

\noindent \textit{Restartation}. We further analyze the complexity bounds of restarting IPPA, that facilitates the derivation of better computational complexity estimates than the non-restarted IPPA. The complexity estimates have similar orders in both restarted and non-restarted algorithms for all $\gamma \ge 1$.

\vspace{5pt}

\noindent \textit{Total computational complexity}. 
We derive total complexity, including the inner computational effort spent at each IPPA iteration, in terms of number of inner (proximal) gradient evaluations.
If $f$ has $\nu-$Holder continuous gradients we obtain that, in the case of best parameter choice, there are necessary:
\begin{align*}
[\gamma = 1+\nu] \qquad	& \BigO{\log(1/\epsilon)}  \\
[\nu = 1] \qquad	&   \BigO{1/\epsilon^{\gamma - 2}}\\
[\nu = 0] \qquad	&   \BigO{1/\epsilon^{2(\gamma - 1)}}	
\end{align*}	
proximal (sub)gradient evaluations to approach to $\epsilon$ distance to the optimal set. As we discuss in the section \ref{sec:total_complexity}, the total complexity is dependent on various restartation variables. 

\vspace{5pt}

\noindent \textit{Experiments}. Our numerical experiments confirm a better behaviour of the restarted IPPA, that uses an inner subgradient method routine, in comparison with other two restarting strategies of classical Subgradient Method. We performed our tests on several polyhedral learning models that includes Graph SVM and Matrix Completion, using synthetic and real data.

\subsection{Notations and preliminaries}\label{sec:prelim}

\noindent Now we introduce the main notations of our manuscript. For $x,y \in \rset^n$ denote the scalar product  $\langle x,y \rangle = x^T y$ and Euclidean norm by $\|x\|=\sqrt{x^T x}$. The projection operator onto set $X$ is denoted by $\pi_X$ and the distance from $x$ to the set $X$ is denoted $\text{dist}_X(x) = \min_{z \in X} \norm{x-z}$. The indicator function of $Q$ is denoted by $\iota_{Q}$. 
Given function $h$, then by $h^{(k)}$ we denote the composition $h^{(k)}(x):= \underbrace{\left(h \circ h \circ \cdots h\right) }_{k \; \text{times}} (x)$.
We use $\partial h(x)$ for the subdifferential set and $h'(x)$ for a subgradient  of $h$ at $x$. In differentiable case, when $\partial h$ is a singleton, $\nabla h$ will be eventually used instead of $h'$. By $X^*$ we denote the optimal set associated to \eqref{problem_of_interest} and by \textit{$\epsilon-$ suboptimal point} we understand a point $x$ that satisfies $\dist_{X^*}(x) \le \epsilon$.

\noindent A function $f$ is called $\sigma-$strongly convex if the following relation holds:
$$ f(x) \ge f(y) + \langle f'(y), x - y\rangle + \frac{\sigma}{2}\norm{x-y}^2 \qquad \forall x,y \in \rset^n. $$
Let $\nu \in [0,1]$, then we say that a differentiable function $f$ has $\nu-$Holder continuous gradient with constant $L>0$ if :
\begin{align*}
\norm{f'(x) - f'(y)} \le L\norm{x-y}^{\nu} \quad
\forall x,y \in \rset^n. 
\end{align*}

\noindent Notice that when $\nu = 0$, the Holder continuity describes nonsmooth functions with bounded gradients, i.e. $\norm{f'(x)} \le L $ for all $x \in \dom(f)$. The $1-$Holder continuity reduces to $L-$Lipschitz gradient continuity. 

\noindent Given a convex function $f$, we denote its Moreau envelope \cite{Roc:76,Ber:89parallel,Roc:76augmented} with $f_{\mu}$ and its proximal operator with $\prox_{\mu}^f(x)$, defined by:
\begin{align*}
f_{\mu}(x) & = \min\limits_{z} \; f(z) + \frac{1}{2\mu}\norm{z -  x}^2 \\ 	
\prox_{\mu}^f(x) &= \arg\min_z \; f(z) + \frac{1}{2\mu}\norm{z -  x}^2. \end{align*} 
We recall the nonexpansiveness property of the proximal mapping \cite{Roc:76}:
\begin{align}\label{rel:prox_nonexpansiveness}
	\norm{ \prox_{\mu}^f(x) - \prox_{\mu}^f(y) } \le \norm{x-y} \quad \forall x,y \in \dom(f). 
\end{align} 
Basic arguments from \cite{Roc:76,Ber:89parallel,Roc:76augmented} show that the gradient $\nabla f_{\mu}$ is Lipschitz continuous with constant $\frac{1}{\mu}$ and satisfies:
\begin{align}\label{rel:smooth_grad_to_nonsmooth_grad}
	\nabla f_{\mu}(x) = \frac{1}{\mu}\left( x - \prox_{\mu}^f(x)\right) \in \partial f(\prox_{\mu}^f(x)). 
\end{align}
In the differentiable case, obviously $\nabla f_{\mu}(x) = \nabla f(\prox_{\mu}^f(x))$. 

\vspace{5pt}

\noindent \textit{Paper structure}. 
In section \ref{sec:Holder_growth} we analyze how the growth properties of $F$ are also inherited by $F_{\mu}$. The key relations on $F_{\mu}$ will become the basis for the complexity analysis.
In section \ref{sec:IPPA} we define the iteration of inexact Proximal Point algorithm and discuss its stopping criterion. The iteration complexity is presented in section \ref{sec:IPPA_complexity} for both the exact and inexact case. Subsequently, the restarted IPPA is defined and its complexity is presented. Finally, in section \ref{sec:total_complexity} we quantify the complexity of IPPA in terms of proximal (sub)gradient iterations and compare with other results. In the last section we compare our scheme with the state-of-the-art restarting subgradient algorithms.

\section{Holderian growth and Moreau envelopes}\label{sec:Holder_growth}

\noindent As discussed in the introduction, $\gamma$-HG relates tightly with  widely known regularity properties such as WSM \cite{Yan:18,Pol:78,Pol:87book,BurFer:93,Fer:91,Ant:94,Dav:18}, Quadratic Growth (QG) \cite{Yan:18,Lu:20,Nec:19}, Error Bound \cite{Luo:93}  and Kurdika-Lojasiewicz inequality \cite{BolNgu:17,Yan:18}. 

\noindent Next we show how the Moreau envelope of a given convex function inherits its growth properties over its entire domain excepting a certain neighborhood of the optimal set. Recall that $\min_x F(x) = \min_x F_{\mu}(x)$.


\begin{lemma}\label{lemma:deterministic_moreau_growth}
Let $F$ be a convex function and let $\gamma-$(HG) hold. Then the Moreau envelope $F_{\mu}$ satisfies the relations presented below.

\noindent $(i)$ Let $\gamma = 1$ WSM:
\begin{align*}
F_{\mu}(x) - F^* \ge H_{\sigma^2_F\mu}(\sigma_F \dist_{X^*}(x)),	
\end{align*}
where $H_{\tau}(s) = \begin{cases}s  - \frac{\tau}{2}, & s > \tau \\ 
	\frac{1}{2\tau}s^2, & s \le \tau
\end{cases}$ is the Huber function.

\noindent $(ii)$ Let $\gamma = 2$:
\begin{align*}
	F_{\mu}(x) - F^* \ge 		\frac{\sigma_F}{1 + 2\sigma_F \mu}\dist_{X^*}^2(x).
\end{align*}

\noindent $(iii)$ For all $\gamma \ge 1$:
\begin{align*}
	F_{\mu}(x) - F^* \ge\varphi(\gamma) \min \left\{   \sigma_F\dist^{\gamma}_{X^*}(x), \frac{1}{2\mu}\dist^{2}_{X^*}(x) \right\},
\end{align*}
where $\varphi(\gamma) =  \min_{\lambda \in [0,1]} \lambda^{\gamma} + (1-\lambda)^2$.

\end{lemma}
\begin{proof}
The proof can be found in the Appendix.
\end{proof}

\noindent  It is interesting to remark that the Moreau envelope $F_{\mu}$ inherits a similar growth landscape as $F$ outside a given neighborhood of the optimal set. For instance, under WSM, outside the tube $\N(\sigma_F\mu) = \{x \in \rset^n: \; \dist_{X^*}(x) \le \sigma_F\mu \}\}$, the Moreau envelope $F_{\mu}$ grows sharply.  
Inside of $\mathcal{N}(\sigma_F \mu)$ it grows quadratically which, unlike the objective function $F$, allows the gradient to get small near to the optimal set. This separation of growth regimes suggests that first-order algorithms that minimize $F_{\mu}$ would reach very fast the region $\mathcal{N}(\sigma_F\mu)$, allowing large steps in a first phase and subsequently, they would slow down in the vicinity of the optimal set.

This discussion extends to general growths when $\gamma > 1$, where a similar separation of behaviours holds for appropriate neighborhoods. Note that when $F$ has quadratic growth with constant $\sigma_F$, also the envelope $F_{\mu}$ satisfies a quadratic growth with a smaller modulus $\frac{\sigma_F}{1 + 2\sigma_F \mu}$. 

  	

\begin{remark} 
It will be useful in the subsequent sections to recall the connection between Holderian growth and Holderian error bound under convexity.
\noindent Observe that by a simple use of convexity into $\gamma$-HG, we obtain for all $x \in \dom F$ 
\begin{align}\label{Holder_error_bound_steps}
\sigma_F \dist_{X^*}^{\gamma}(x) \le F(x) - F^* \le \langle F'(x), x - \pi_{X^*}(x^*) \rangle  \le \norm{F'(x)}\dist_{X^*}(x),
\end{align}
which immediately turns into the following error bound:
\begin{align}\label{Holder_error_bound}
\sigma_F \dist_{X^*}^{\gamma-1}(x) \le  \norm{F'(x)} 	\qquad x \in \dom{F}.
\end{align}
Under WSM, by replacing $x$ with $\prox_{\mu}^F(x)$ into \eqref{Holder_error_bound} and by using property \eqref{rel:smooth_grad_to_nonsmooth_grad} pointing that $ \nabla F_{\mu}(x) \in \partial F(\prox_{\mu}^F(x))$, we obtain as well a similar bound on $\norm{\nabla F_{\mu}(\cdot)}$ at non-optimal points:
	\begin{align}\label{gamma1_Env_Holder_error_bound}
		\sigma_F \le \norm{\nabla F_{\mu}(x)}  \quad \forall x \notin X^*.
	\end{align}
This is the traditional key relation for the finite convergence of PPA. Under $\gamma-$HG, starting from Lemma \ref{lemma:deterministic_moreau_growth}$(iii)$, by using convexity of $F_{\mu}$ and, further, by following similar inequalities as in \eqref{Holder_error_bound_steps}, another error bound is obtained:   
\begin{align}\label{gamma_all_Env_Holder_error_bound}
\dist_{X^*}(x)	 \le \Max{ \left[ \frac{1}{\sigma_F \varphi(\gamma)}\norm{\nabla F_{\mu}(x)} \right]^{\frac{1}{\gamma - 1}} , \frac{2\mu}{\varphi(\gamma)}\norm{\nabla F_{\mu}(x)}   } \quad \forall x.
\end{align}
\end{remark}

\noindent In the following section we start with the analysis of exact and inexact PPA. Aligning to old results on the subgradient algorithms back to \cite{Pol:78}, we illustrate the robustness induced by the weak sharp minima regularity.


\section{Inexact Proximal Point algorithm}\label{sec:IPPA}



\noindent The basic exact PPA iteration is shortly described as 
\begin{align*}
x^{k+1} = \prox_{\mu}^F(x^k).
\end{align*} 
Recall that by \eqref{rel:smooth_grad_to_nonsmooth_grad} one can express $\nabla F_{\mu}(x^k) = \frac{1}{\mu}(x^k - \prox_{\mu}^F(x^k))$, which makes PPA equivalent with the constant stepsize Gradient Method (GM) iteration: 
\begin{align}\label{Gradient_Method_PPA}
x^{k+1} = x^k - \mu \nabla F_{\mu}(x^k).
\end{align}
Since our reasoning from below borrow simple arguments from classical GM analysis, we will use further \eqref{Gradient_Method_PPA} to express PPA.
It is realistic not to rely on explicit $\prox_{\mu}^F(x^k)$, but an approximated one to a fixed accuracy. By using such an approximation, one can form an approximate gradient $\nabla_{\delta} F_{\mu}(x^k)$ and interpret IPPA as an inexact Gradient Method. Let $x \in \dom{F}$, then a basic $\delta-$ approximation of $\nabla F_{\mu}(x)$ is  
\begin{align}\label{inexactness_criterion}
\nabla_{\delta} F_{\mu}(x):=\frac{1}{\mu}\left(x  - \tilde{z} \right), \quad \text{where} \quad \norm{\tilde{z} - \prox_{\mu}^F(x)} \le \delta.
\end{align}

\noindent Other works as \cite{Roc:76,Sal:12,Sol:00,Sol:01} promotes similar approximation measures for inexact first order methods. 
\noindent Now we present the basic IPPA scheme with constant stepsize.


\begin{algorithm}
	\caption{Inexact Proximal Point Algorithm($x^0,\mu,\{\delta_k\}_{k \ge 0}, \epsilon$)}
	Initialize $\; k: = 0$\\
	\While{  $\norm{\nabla_{\delta_k} F(x^k)} > \epsilon$ or $\delta_k > \frac{\epsilon}{\mu}$}{
		$\text{Given $x^k$ compute} \;\; x^{k+1} \; \text{such that}: \; \norm{x^{k+1} - \prox_{\mu}^F(x^k)} \le \delta_k$ \\
		$k := k+1 $	\\
	}
	Return $x^k$ 
\end{algorithm}
\noindent There already exist a variety of relative or absolute stopping rules in the literature for the class of first-order methods \cite{Tom:11,Sal:12,Pol:87book,Hum:05}. 
However, since bounding the norms of gradients is commonly regarded as one of the simplest optimality measures, we use the following: 
\begin{align}\label{stopping_criterion}
\norm{\nabla_{\delta_k} F(x^k)} < \epsilon \qquad \delta_k \le \frac{\epsilon}{\mu}
\end{align}
which is computable by the nature of the iteration. The next result shows the relation between \eqref{stopping_criterion} and the distance to optimal set.

\begin{lemma}\label{lemma:morenvgrad_bound_relating_with_distopt}
Let $\mu,\delta > 0$ and assume $x \notin X^*$ then:

\noindent $(i)$ Let $\gamma = 1$, if $\norm{\nabla_{\delta} F_{\mu}(x^k)}  + \frac{\delta}{\mu}< \sigma_F$, then 
\begin{align}\label{direct_distance_bound}
\dist_{X^*}(x^k) \le 
\mu \norm{\nabla_{\delta} F_{\mu}(x^k)}  + \delta \quad \text{and} \quad \prox_{\mu}^F(x^k) = \pi_{X^*}(x^k).
\end{align}

\noindent $(ii)$ Let $\gamma > 1$, then:
$\;\;	\dist_{X^*}(x^k) \le 
		\Max{\left[ \frac{\mu \norm{\nabla_{\delta} F_{\mu}(x^k)}  + \delta}{\mu\sigma_F \varphi(\gamma)} \right]^{\frac{1}{\gamma  - 1}}, \frac{2\mu \norm{\nabla_{\delta} F_{\mu}(x^k)}  + \delta}{\varphi(\gamma)}}.$
\end{lemma}
\begin{proof}
The proof can be found in the Appendix.
\end{proof}

\noindent If $X^*$ are weakly sharp minimizers, the above lemma states that, for sufficiently small $\delta$ and $\norm{\nabla_{\delta } F_{\mu}(x^k)}$ one can directly guarantee that $x^k$ is close to $X^*$. From another viewpoint, Lemma \ref{lemma:morenvgrad_bound_relating_with_distopt} also suggests for WSM that a sufficiently large $\mu > \frac{\dist_{X^*}(x^0)}{\sigma_F}$ provides $\prox_{X^*}(x^0) = \pi_{X^*}(x^0)$ and a $\delta-$optimal solution is obtained as the output of the first IPPA iteration. A similar result stated in \cite{Fer:91}, guarantees the existence of a sufficiently large smoothing value $\mu$ which makes PPA to converge in a single (exact) iteration.


\section{Iteration complexity of IPPA}\label{sec:IPPA_complexity}

For reasons that will be clear later, note that some of the results from below are given for constant inexactness noise $\delta_k = \delta$. The Holderian growth leads naturally to recurrences on the residual distance to the optimal set, that allow us to give direct complexity bounds.

%
%

\begin{theorem}\label{th:IPPconvergence}
Let $F$ be convex and $\gamma-$HG hold.  

\noindent $(i)$ Under sharp minima $\gamma = 1$, let $\{\delta_k\}_{k \ge 0}$ nonincreasing and assume $\dist_{X^*}(x^{0}) \ge \mu \sigma_F$, then:
\begin{align*}
\;\; \dist_{X^*}(x^{k}) 
\le  \max\left\{\dist_{X^*}(x^{0}) - \sum\limits_{i=0}^{k-1}(\mu\sigma_F - \delta_i), \delta_{k-1}  \right\}. 
\end{align*}
	
\noindent $(ii)$ Under quadratic growth $\gamma = 2$, let $\sum\limits_{i \ge  0} \delta_i < \Gamma < \infty$  then:
\begin{align*}
\dist_{X^*}(x^{k}) 
\le   \left[\frac{1}{1 + 2\mu \sigma_F}\right]^{\frac{k-4}{4}} ( \dist_{X^*}(x^{0}) +  \Gamma )  + \left(1 + \frac{1}{\sqrt{1+2\mu\sigma_F}-1}\right) \delta_{\lceil \frac{k}{2} \rceil + 1} .
\end{align*}
	
\noindent $(iii)$ Let $\gamma-$HG hold. 
Then the following convergence rate holds:
\begin{align*}
\dist_{X^*}(x^k) 
	&\le 
		\max\left\{h^{(k)}(\dist_{X^*}(x^0)), \bar{\delta}_k \right\}
\end{align*}	
where $ h(r) = 
\begin{cases}
	\Max{r - \frac{\mu\varphi(\gamma)\sigma_F}{2} r^{\gamma-1}, \frac{1+\sqrt{1-\varphi(\gamma)}}{2}r }, &  \text{if} \; \gamma \in (1,2) \\ 
	\Max{r - \frac{\mu\varphi(\gamma)\sigma_F}{2} r^{\gamma-1}, \frac{1+\sqrt{1-\varphi(\gamma)}}{2}r ,\left( 1- \frac{1}{\gamma - 1}\right)r}, &  \text{if} \; \gamma > 2.
\end{cases}$
 and
$\bar{\delta}_k = \Max{  h(\bar{\delta}_{k-1}), \left(\frac{2\delta_k}{\mu\varphi(\gamma)\sigma_F}\right)^{\frac{1}{\gamma -1}},\frac{2\delta_k}{1-\sqrt{1-\varphi(\gamma)}}   }.$	
\end{theorem}
\begin{proof}
The proof can be found in the appendix.
\end{proof}	

\noindent Note that in general, all the abstract convergence rates of Theorem \ref{th:IPPconvergence} depend on two terms: the first one illustrates the reduction of the initial distance to optimum and the second term reflects the accuracy of the approximated gradient. Therefore, after a finite number (for $\gamma = 1$) or at most $\mathcal{O}(\log(1/\epsilon))$ (for $\gamma > 1$) IPPA iterations, the evolution of inner accuracy $\delta_k$ becomes the main bottleneck of the convergence process. 
 The above theorem provides an abstract insight about how fast should the accuracy $\{\delta_k\}_{k \ge 0}$ decay in order that  $\{x^k\}_{k \ge 0}$ attains the best rate towards the optimal set in terms of $\dist_{X^*}(x^k)$. In short, $(i)$ shows that for WSM a recurrent constant decrease on $\dist_{X^*}(x^k)$ is established only if $\delta_k < \mu\sigma_F$, while the noise $\delta_k$ is not necessary to vanish. This aspect will be discussed in more details below.  The last parts $(ii)$ and $(iii)$ suggest that a linear decay  of $\delta_k$ (for $\gamma = 2$) and, respectively, $\delta_{k+1} = h(\delta_k)$ for general $\gamma > 1$, ensure the fastest convergence of the residual. 

\noindent Several previous works as \cite{Ned:10,Pol:78,Pol:87book} analyzed perturbed and incremental SGM algorithms, under WSM, that use noisy estimates $G_k$ of subgradient $F'(x^k)$. A surprising common conclusion of these works is that under a sufficiently low persistent noise:
$0< \norm{F'(x^k) - G_k} < \sigma_F,$ for all $k \ge 0$,	
SGM still converges linearly to the optimal set. Although IPPA is based on a similar approximate first order information, notice that the smoothed objective $F_{\mu}$ do not satisfy the pure WSM property. However, our next result states, in a similar context of small but persistent noise of magnitude at most $\frac{\delta}{\mu}$, that IPPA attains $\delta-$accuracy after a finite number of iterations.


\begin{corrolary}\label{cor:complexity_sharp_minima}
Let $\delta_k = \delta < \mu\sigma_F$ and $\gamma = 1$, then after
\begin{align*}
K = \left\lceil \frac{\dist_{X^*}(x^0)}{\mu\sigma_F - \delta}\right\rceil  
\end{align*}
IPPA iterations, $x^K$ satisfies $\prox_{\mu}^F(x^K) = \pi_{X^*}(x^K)$ and $\dist_{X^*}(x^K) \le \delta$.
\end{corrolary}
\begin{proof}
The proof can be found in the Appendix.
\end{proof}

\noindent To conclude, if the noise magnitude $\frac{\delta}{\mu}$ is below the threshold $\sigma_F$, or equivalently 	$0< \delta^{\nabla} : =\norm{\nabla F_{\mu}(x^k) - \nabla_{\delta} F_{\mu}(x^k)} < \sigma_F	$	
then after a finite number of iterations IPPA reaches $\delta$ distance to $X^*$.
We see that under sufficiently low persistent noise, IPPA still guarantees convergence to the optimal set assuming the existence of an inner routine that computes each iteration.  In other words, "noisy" Proximal Point algorithms share similar stability properties as the noisy Subgradient schemes from \cite{Ned:10,Pol:78,Pol:87book}.
This discussion can be extended to general decreasing $\{\delta_k\}_{k \ge 0}$. 

\noindent We show next that Theorem \ref{th:IPPconvergence} covers well-known results on exact PPA. 

\begin{corrolary}\label{cor:complexity_for_exact_case}
Let $\{x^k\}_{k \ge 0}$ be the sequence of exact PPA, i.e. $\delta_k = 0$. By denoting $r_0 = \dist_{X^*}(x^0)$, an $\epsilon-$suboptimal iterate is attained, i.e. $\dist_{X^*}(x^k) \le \epsilon$, after the number of iterations:
\begin{align*}	
	\K_e(\gamma,\epsilon) = 
	\begin{cases}
		\left\lceil \frac{r_0 - \epsilon}{\mu\sigma_F}\right\rceil &\text{if} \; \gamma = 1 \\
		\left\lceil \mathcal{O}\left( \frac{1}{\mu\sigma_F} \log\left(\frac{r_0}{\epsilon}\right) \right) \right\rceil
		&\text{if} \; \gamma = 2 \\
		\left\lceil \mathcal{O}\left(   \Min{  \T^{2-\gamma} \log\left( \frac{r_0 }{ \epsilon   }       \right)    ,   \T    }\right) \right\rceil, &\text{if} \; \gamma \in (1,2), \; \epsilon \ge (\mu\varphi(\gamma)\sigma_F )^{\frac{1}{2-\gamma} }  \\
		\left\lceil \mathcal{O}\left(   \log\left(  \Min{r_0, (2\mu\sigma_F)^{\frac{1}{2-\gamma}}}  / \epsilon    \right)     \right) \right\rceil	
		&\text{if} \; \gamma \in (1,2), \epsilon  < (\mu\varphi(\gamma)\sigma_F )^{\frac{1}{2-\gamma} }\\
		\left\lceil \BigO{    \frac{1}{\epsilon^{\gamma - 2}}    }\right\rceil 		&\text{if} \; \gamma > 2. 
	\end{cases}
\end{align*}

%
\end{corrolary}
\begin{proof}
The proof can be found in the Appendix.
\end{proof}

\noindent The finite convergence of the exact PPA, under WSM, dates back to \cite{Fer:91,BurFer:93,Ant:94,Ber:82constrained}. 
Since PPA is simply a gradient descent iteration, its iteration complexity under QG $\gamma = 2$ shares the typical dependence on the conditioning number $\frac{1}{\mu\sigma_F}$.
The Holder growth $\gamma \in (1,2)$ behaves similarly with the sharp minima case: fast convergence outside the neighborhood around the optimum which expands with $\mu$. 

\noindent A simple argument on the tightness of the bounds for the case $\gamma > 2$ can be found in \cite{BauDao:15}. Indeed, Douglas-Rachford, PPA and Alternating Projections algorithms were analyzed in \cite{BauDao:15} for particular univariate functions. The authors proved that PPA requires $\BigO{1/\epsilon^{\gamma -2}}$ iterations to minimize the particular objective $F(x) = \frac{1}{\gamma}|x|^{\gamma}$ (when $\gamma > 2$) up to $\epsilon$ tolerance.

\begin{corrolary}\label{cor:complexity_for_inexact_case}
Under the assumptions of Corrolary \ref{cor:complexity_for_exact_case}, recall the notation $\K_e(\gamma,\epsilon)$ for the exact case. The number of iterations performed by $\{x^k\}_{k \ge 0}$ generated with IPPA$(x^0,\mu,\{\delta_k\})$ in order to attain an $\epsilon-$suboptimal point is:
\begin{align*}
\begin{cases}
\left\lceil \BigO{\Max{\K_e(1,\epsilon)+\frac{\delta_0}{\mu\sigma_F}, \log\left( \frac{\delta_0}{\epsilon}\right)}   } \right\rceil  & \quad \gamma = 1, \delta_k \le \frac{\delta_0}{2^k} \\
\left\lceil  \mathcal{O}\left( \K_e(\gamma,\epsilon)\right) \right\rceil 
& \quad \gamma \in (1,2], \delta_k \le \frac{\delta_0}{2^{k}} \\
\left\lceil  \mathcal{O}\left( \K_e(\gamma,\epsilon)\right) \right\rceil 
& \quad
\gamma > 2, \delta_k = \left( 1 - \frac{1}{\gamma - 1}\right)^{k(\gamma - 1)} \delta_{0}.
\end{cases}
\end{align*}
\end{corrolary}
\begin{proof}
The proof can be found in the Appendix.
\end{proof}
	
\noindent Overall,  if the noise vanishes sufficiently fast (e.g. a linear decay with factor $\frac{1}{2}$ in the case $\gamma \le 2$) then the complexity of IPPA has the same order as the one of PPA. However, a fast vanishing noise $\delta_k$ requires an efficient inner routine that computes the iterate $x^k$. Particularly, when $F$ is nonsmooth and convex, a simple choice of the inner routine is the Subgradient Method (SM). However, the efficiency of the SM for minimizing a given $\sigma_f-$strongly convex function $f$ up to $\delta$ accuracy, i.e. $f(x) - f^* \le \delta$, has the order $\BigO{\frac{1}{\delta}}$ subgradient calls \cite{JudNes:14uniform}. By using the quadratic growth guaranteed by strong convexity, i.e. $\frac{\sigma_f}{2}\norm{x-x^*}^2 \le f(x) - f^* \le \delta$,  in order to approach a minimizer at distance $\delta$, SM requires  $\BigO{\frac{1}{\delta^2}}$ iterations.
According to this bound, the cost of computing a single iteration of IPPA results into $\BigO{\frac{1}{\delta_k^2}}$ subgradient calls and, therefore, a direct naive counting of the total number of subgradient evaluations over $T$ outer iterations is of order $\BigO{\sum\limits_{k=0}^T \frac{1}{\delta_k^2}}$. By using the previous estimates of $T$ necessary for $\dist_{X^*}(x^T) \le \epsilon$, yields a total estimate of subgradient calls which is significantly larger than the known optimal bound of $\BigO{\frac{1}{\epsilon^{2(\gamma-1)}}}$ for SM algorithm. However, we further give a restarted variant of IPPA that overcomes this issue.

\section{Restarted Inexact Proximal Point Algorithm}\label{sec:RIPPA}

\noindent The Restarted IPPA (RIPPA) illustrates a simple recursive call of the IPPA combined with a multiplicative decrease of the parameters. Observe that RIPPA is completely independent of the problem constants.

\begin{algorithm}
	\caption{Restarted IPPA ($x^0,\mu_0,\delta_0,\rho$)}
	Initialize $\delta^{\nabla}_0 := \delta_0, t:=0$\\
	\While{stopping criterion}{
		$\text{Call IPPA}  \; \text{to compute:} \;\; x^{t+1} = IPPA(x^t, \mu_t, \{\delta_k := \delta_t\}_{k \ge 0}, 5\delta_t^{\nabla} )$ \\
		$\text{Update:} \;\; \mu_{t+1} = 2\mu_t, \; \delta^{\nabla}_{t+1} = \frac{\delta^{\nabla}_{t}}{2^{\rho}} , \; \delta_{t+1} = \mu_{t+1}\delta^{\nabla}_{t+1}$	\\
		$\; t := t+1$\\
		
	}
\end{algorithm}

As in the usual context of restartation, we call any $t-$th iteration an epoch. The stopping criterion can be optionally based on a fixed number of epochs or on the reduction of gradient norm \eqref{stopping_criterion}.
Denote $K_0 = \lceil  \frac{\dist_{X^*}(x^0)}{\mu_0\delta_0} \rceil $.

\begin{theorem}\label{th:RIPP_complexity}
	Let $\delta_0,\mu_0$ be positive constants and $\rho > 1$. Then the sequence $\{x^t\}_{t\ge 0}$ generated by $RIPPA(x^0,\mu_0,\delta_0)$ attains $\dist_{X^*}(x^t) \le \epsilon$ after a number of $\T_{IPP}(\gamma,\epsilon)$ iterations. 
	
	\noindent Let $\gamma = 1$ and assume $\epsilon < \mu_0\sigma_F$ and $\dist_{X^*}(x^0) \ge \mu_0\sigma_F$, then
	\begin{align*}
		\T_{IPP}(1,\epsilon) =  \frac{1}{\rho - 1} \log\left( \frac{\mu_0 \delta_0}{\epsilon} \right) + \T_{ct},
	\end{align*}
	where $\T_{ct} = K_0\lceil \frac{1}{\rho}\log\left( \frac{12\delta_0}{\sigma_F} \right) \rceil .$	
	In particular, if $\delta_0 < \mu_0\sigma_F$ then $RIPPA(x^0,\mu_0,\delta_0,0)$ reaches the $\epsilon-$suboptimality within $\BigO{\log\left( \frac{\mu_0\sigma_F}{\epsilon}\right)}$ iterations.
	
	\vspace{5pt}
	
	\noindent Let $\gamma = 2$, then 
	\begin{align*}
		\T_{IPP}(2,\epsilon) =  \BigO{\frac{1}{\rho}\log \left( \frac{\delta_0}{\epsilon}\right) } + K_0,
	\end{align*}
	
	\vspace{5pt}

	\noindent Let $\gamma \in (1,2)$, then:
	\begin{align*}
		\T_{IPP}(\gamma,\epsilon) = \BigO{\Max{ \frac{\gamma-1}{\rho}\log \left( \frac{\delta_0}{\epsilon}\right), \frac{1}{\rho-1}\log \left( \frac{\mu_0\delta_0}{\epsilon}\right)}} + K_0.
	\end{align*}
	Otherwise, for $\gamma > 2$
	\begin{align*}
		\T_{IPP}(\gamma,\epsilon) = \BigO{\left( \frac{\delta_0}{\epsilon}\right)^{ \left[\left(1  -\frac{1}{\rho} \right)(\gamma - 1) - 1 \right]\Max{ 1 , \frac{1}{(1 - 1/\rho)(\gamma - 1)} }    } } + K_0.
	\end{align*}
\end{theorem}
\begin{proof}
The proof can be found in the Appendix.
\end{proof}

\begin{remark}
	Notice that for any $\gamma \in [1,2]$, logarithmic complexity $\BigO{\log(1/\epsilon)}$ is obtained.
	When $\gamma > 2$,  the above estimate is shortened as
	\begin{align*}
		\BigO{\left( \frac{1}{\epsilon}\right)^{(\zeta - 1)\Max{1,\frac{1}{\zeta}}}},
	\end{align*}
	where $\zeta = (\gamma - 1)\left( 1 - \frac{1}{\rho}\right)$,
	In particular, if $\rho \le \frac{\gamma - 1}{\gamma - 2}$, then all epochs reduce to length $1$ and the total number of IPPA iterations reduces to the same order as in the exact case:
	$\;	\BigO{\left( \frac{1}{\epsilon}\right)^{\gamma - 2} }.$
\end{remark}

\section{Inner Proximal Subgradient Method routine}\label{sec:total_complexity}

Although the influence of growth modulus on the behaviour of IPPA is obvious, all complexity estimates derived in the previous sections assume the existence of an oracle computing an approximate proximal mapping:
\begin{align}\label{the_subproblem}
x^{k+1} \approx	\arg\min\limits_{z \in \rset^n} \; F(z) + \frac{1}{2\mu}\norm{z - x^k}^2.
\end{align}
Therefore in the rest of this section we focus on the solution of the following inner minimization problem:
\begin{align*}
	\min\limits_{z \in \rset^n} \; f(z) + \psi(z) + \frac{1}{2\mu}\norm{z - x}^2.
\end{align*}
In most situations, despite the regularization term, this problem is not trivial and one should select an appropriate routine that computes $\{x^k\}_{k \ge 0}$.  
For instance, variance-reduced or accelerated proximal first-order methods were employed in \cite{Lin:15,Lin:18,Lu:20,Luo:93} as inner routines and theoretical guarantees were provided. Also,  Conjugate Gradients based Newton method was used in \cite{Tom:11} under a twice differentiability assumption on $f$. We limit our analysis only to gradient-type routines and let other accelerated or higher-order methods, that typically improve the performance of their classical counterparts, for future work.

In this section we evaluate the computational complexity of IPPA in terms of number of proximal gradient iterations. The basic routine for solving \eqref{the_subproblem}, that we analyze below, is the Proximal (sub)Gradient Method. Notice that when $f$ is nonsmooth with bounded subgradients, we consider only the case when $\psi$ is an indicator function for a simple, closed convex set. In this situation, PsGM becomes a simple projected subgradient scheme with constant stepsize that solves \eqref{the_subproblem}.

\begin{algorithm}
	\caption{Proximal subGradient Method (PsGM) ($z^0,x^k, \alpha, \mu, N$)}
	\For{$\ell = 0, \cdots, N-1 $}{
		$z^{\ell+1} = \prox_{\alpha}^\psi \left(z^\ell - \alpha \left( f'(z^\ell) + \frac{1}{\mu}(z^\ell - x^k)\right) \right) $	\\
	}
	$\text{Output:} \; z^N$
\end{algorithm}

Through a natural combination of the outer guiding IPP iteration and the inner PsGM scheme, we further derive the total complexity of proximal-based restartation Subgradient Method in terms of subgradient oracle calls.


\begin{algorithm}\label{algorithm:IPP-SGM}
	\caption{Restarted Inexact Proximal Point - SubGradient Method (RIPP-PsGM) ($x^0, \delta_0, \mu_0, \rho, q, L_f, T$)}
	Initialize: $t := 0, \alpha_0 : = \frac{\mu_0}{2}, N_0 : = \Max{8\log(L_f/\delta_0)+1,\rho-1} $  \\
	\For{ $t:=0, \cdots, T-1$}{
		$x^{0,t} : = x^t, k: = 0$\\
		\Do{ $\norm{x^{k-1,t} - x^{k-2,t}} > \mu_t\delta_t$ }{
		$ x^{k+1,t} := PsGM(x^{k,t}, x^{k,t}, \alpha_t, \mu_t, N_t)$ \\
		$k:=k+1$ \\
	} 
		$x^{t+1}:= x^{k-1,t}$\\	
	{$\alpha_{t+1}: = \alpha_{t}2^{-q}, N_{t + 1}: = N_t 2^{-(q+1)}$\\
	$\mu_{t+1}:= 2\mu_t, \delta_{t+1}:= 2^{-\rho}\delta_t $\\
	}
}
Return $x^{T}, \delta_T$ \\

\vspace{5pt}

\hrule

\vspace{5pt}

\textbf{Postprocessing}($\tilde{x}^0,\beta_0,\mu,N,K$)  \\
\For{ $k:=0, \cdots, K-1$}{
		$ \tilde{x}^{k+1} := PsGM(\tilde{x}^{k}, \tilde{x}^{k}, \beta_k, \mu, N)$ \\
		$\beta_{k+1}: = \beta_{k}/2$}
Return $\tilde{x}^{K}$
\end{algorithm}

\begin{theorem}\label{th:main_computational_theorem}
Let the assumptions of Theorem \ref{th:Holder_gradients} hold and $\rho > 1, \delta_0 \ge 2L_f$ then the following assertions hold:

\vspace{5pt}

\noindent $(i)$ Let $x^f$ be generated as follows:
\begin{align*}
\{x^T, \delta_T\} &= \textbf{RIPP-PsGM}(x^0,\delta_0,\mu_0,\rho,2\rho - 1,L_f,T) \\ 
x^f &= \textbf{Postprocessing}(x^T,\delta_T^2/L_f^2 ,\mu_T,\lceil 2(L_f/\delta_T)^2 \rceil,\lceil\log(\delta_T/\epsilon)\rceil).  
\end{align*}
For $T \ge \left\lceil \frac{1}{\rho} \log(12\delta_0/\sigma_F) \right \rceil$, the final output $x^f$ satisfies $\dist_{X^*}(x^f) \le \epsilon$ after 
$\BigO{\log\left( 1 / \epsilon \right) }$
PsGM iterations. 

\vspace{5pt}

\noindent $(ii)$ Moreover, if $f$ has $\nu-$Holder continuous gradients with constant $L_f$ and $\nu \le \gamma - 1$, assume $\delta_0 \ge (2L_f^2)^{\frac{1}{2(1-\nu)}} \mu_0^{\frac{\nu}{1-\nu}}$ and
$\frac{q-1}{\rho - 1} \ge 2(1-\nu)$.
If $T = \BigO{  \Max{   \frac{\gamma-1}{\rho}\log(\mu_0\delta_0/\epsilon), \frac{1}{\rho-1}\log(\mu_0\delta_0/\epsilon)  } }$, then the output $x^T:=RIPP-PsGM(x^0,\delta_0,\mu_0,\rho,q,L_f,T)$ satisfies $\dist_{X^*}(x^T) \le \epsilon$  within a total cost of: 
\begin{align*}
	\BigO{ 1/\epsilon^{  \Max{ \gamma - 2 + \frac{q}{\rho}(\gamma - 1) , (\gamma-2)\frac{\rho}{(\rho - 1)(\gamma - 1)} + \frac{q}{\rho - 1}  }     }}
\end{align*}
PsGM iterations.
\end{theorem}
\begin{proof}
The proof can be found in the Appendix.
\end{proof}

\vspace{10pt}

\begin{remark}
Although the bound on the number of epochs in RIPP-PsGM depends somehow on $\gamma$, a sufficiently high value of $T$ ensure the result to hold. To investigate some important particular bounds hidden into the above complexity estimates (of Theorem \ref{th:main_computational_theorem}) we analyze different choices of the input parameters $(\rho,q)$ which will be synthesized in Table 1. 

\noindent Assume $\nu$ is known, $q = 1 + 2(1-\nu)(\rho - 1)$ and denote $\zeta = \left(1 - \frac{1}{\rho} \right)(\gamma - 1)$
\begin{align}
[\gamma = 1 + \nu ] & \quad \BigO{ 1/\epsilon^{   (3-2\gamma)(\zeta - 1)\Max{1,\frac{1}{\zeta}}     }} \label{rel:known_nu_estimates1}\\
[\gamma > 1 + \nu ] & \quad\BigO{ 1/\epsilon^{ \left[ 2(\gamma - \nu - 1) + (1-2\nu)(\zeta - 1) \right]\Max{1,\frac{1}{\zeta}}          }}. \label{rel:known_nu_estimates2}
\end{align}

\noindent Under knowledge of $\gamma$, by setting $\rho = \frac{\gamma - 1}{\gamma - 2}$ then $\zeta = 1$ and \eqref{rel:known_nu_estimates1} becomes $\BigO{\log(1/\epsilon)}$. When $\gamma < 2$, a sufficiently large $\rho$ simplify \eqref{rel:known_nu_estimates2} into $\BigO{1/\epsilon^{3-2\nu - \frac{1}{\gamma - 1}   }}$.
Given any $\nu \in [1/2,1]$ and $\gamma > 2$, similarly for $\rho \to \infty $ the estimate \eqref{rel:known_nu_estimates2} reduces to $\BigO{1/\epsilon^{(3-2\nu)(\gamma - 1)-1}}$. 

\noindent In the particular smooth case $\nu = 1$, bounds \eqref{rel:known_nu_estimates1}-\eqref{rel:known_nu_estimates2} become:
\begin{align*}
	[\gamma = 2 ] & \quad \BigO{ 1/\epsilon^{   \Max{\frac{1}{\rho},\frac{1}{\rho-1}}     }} \\
	[\gamma > 2 ] & \quad\BigO{ 1/\epsilon^{  \left[\gamma - 2 + \frac{\gamma - 1}{\rho} \right]\Max{1,\frac{1}{\zeta}}          }}.
\end{align*}
For high values of $\rho \ge \log(1/\epsilon)$, the first one becomes $\BigO{ \log(1/\epsilon) }$. Also the second one reduces to $\BigO{ 1/\epsilon^{\gamma - 2} }$ when $\rho \ge (\gamma - 1) \log(1/\epsilon)$.

\noindent In the bounded gradients case $\nu = 0 $ the estimates reduces to
\begin{align}
	[\gamma = 1 ] & \quad \BigO{ \log(1/\epsilon)     } \label{rel:no_info_estimates1}\\
	[\gamma > 1 ] & \quad\BigO{ 1/\epsilon^{  \left[2(\gamma - 1) + \zeta - 1 \right]\Max{1,\frac{1}{\zeta}}          }}. \label{rel:no_info_estimates2}
\end{align}
First observe that when the main parameters $\sigma_F,L_f,\gamma$ are known then $\zeta = 1$ and we recover the same iteration complexity in terms of the number of subgradient evaluations as in the literature \cite{Yan:18,Joh:20,Rou:20,Fre:18}. 
The last estimate \eqref{rel:no_info_estimates2} holds when RIPP-PsGM performs a sufficiently high number of epochs, under no availability of problem parameters $\sigma_F, \nu, \gamma$.   
\end{remark}



\noindent  In \cite{NemNes:85,Rou:20} are derived the optimal complexity estimates $\BigO{\epsilon^{-\frac{2(\gamma - \nu-1)}{3(1+\nu) - 2}   }   }$, in terms of the number of (sub)gradient evaluations, for accelerated first-order methods under $\nu-$Holder smoothness and $\gamma-$HG. These optimal estimates require full availability of the problem information: $\gamma, \sigma_F, \nu, L_F$.

\begin{table}[h!]
	\begin{center}
		
		
		
		\begin{tabular}{c|c|c|c|c}
			\textbf{Knowledge} & \textbf{DS-SG}\cite{Joh:20} &  \textbf{RSG}\cite{Yan:18} & \textbf{Restarted UGM}\cite{NemNes:85} & \textbf{IPPA-PsGM} \\
			\textbf{} &  &   &  \\
			\hline
			$\sigma_F, \gamma, L_f (\nu = 0)$ &
			$\BigO{\epsilon^{-2(\gamma - 1)}}$ &
			$\BigO{\epsilon^{-2(\gamma - 1)}}$ &
			$\BigO{\epsilon^{-2(\gamma - 1)}}$ &
			$\BigO{\epsilon^{-2(\gamma - 1)}}$ \\
			
			$\sigma_F, \gamma, L_f, \nu > 0 $ &
			- &
			- &
			$\BigO{\epsilon^{-\frac{2(\gamma - \nu-1)}{3(1+\nu) - 2}}}$ &
			\eqref{rel:known_nu_estimates1}/\eqref{rel:known_nu_estimates2} \\
			
			$\gamma, L_f$ &
			- &
			- &
			$\BigO{\epsilon^{-2(\gamma - 1)}}$ &
			$\BigO{\epsilon^{-2(\gamma - 1)}}$ \\
			
			$\nu, L_f$ &
			- &
			- &
			- &
			\eqref{rel:known_nu_estimates1}/\eqref{rel:known_nu_estimates2} \\
			
			$L_f$ &
			- &
			- &
			- &
			\eqref{rel:no_info_estimates2} 
		\end{tabular}
		\caption{Comparison of complexity estimates under various knowledge degrees of problem information} 
	\end{center}
\label{tab:comparison}
\end{table}

\section{Numerical simulations}\label{sec:numerical}

In the following Section we evaluate RIPP-PsGM by applying it to real-world applications often found in machine learning tasks.
The algorithm and its applications are public and available online~\footnote{https://github.com/pirofti/IPPA}.

Unless stated otherwise,
we perform enough epochs (restarts) until the objective is within $\varepsilon_0=0.5$ proximity to the CVX computed optimum.
The current objective value is computed within each inner PsGM iteration. All the models under consideration satisfy WSM property and therefore the implementation of PsGM reduces to the scheme of (projected) Subgradient Method.



We would like to thank the authors of the methods we compare with for providing the code implementation for reproducing the experiments.
No modifications were performed by us on the algorithms or their specific parameters.
Following their implementation and as is common in the literature,
in our reports we also use the minimum error obtained in the iterations thus far.

All our experiments were implemented in Python 3.10.2 under ArchLinux (Linux version 5.16.15-arch1-1) and executed on an AMD Ryzen Threadripper PRO 3955WX with 16-Cores and 512GB of system memory.

\subsection{Robust $\ell_1$ Least Squares}

We start out with the least-squares (LS) problem in the $\ell_1$ setting.
This form deviates from standard LS by imposing an $\ell_1$-norm on the objective and
by constraining the solution sparsity through the $\tau$ parameter on its $\ell_1$-norm:
\begin{align*}
	\min\limits_{x \in \rset^n} \; & \; \norm{Ax - b}_1 \\
	\text{s.t.} & \;\; \norm{x}_1 \le \tau
\end{align*}

\noindent Our goal is to analyze the effect of the data dimensions, the problem and IPPA parameters
on the total number of iterations.

\begin{figure}[t]
	\centering
	\subfigure[]{\includegraphics[width=0.31\textwidth]{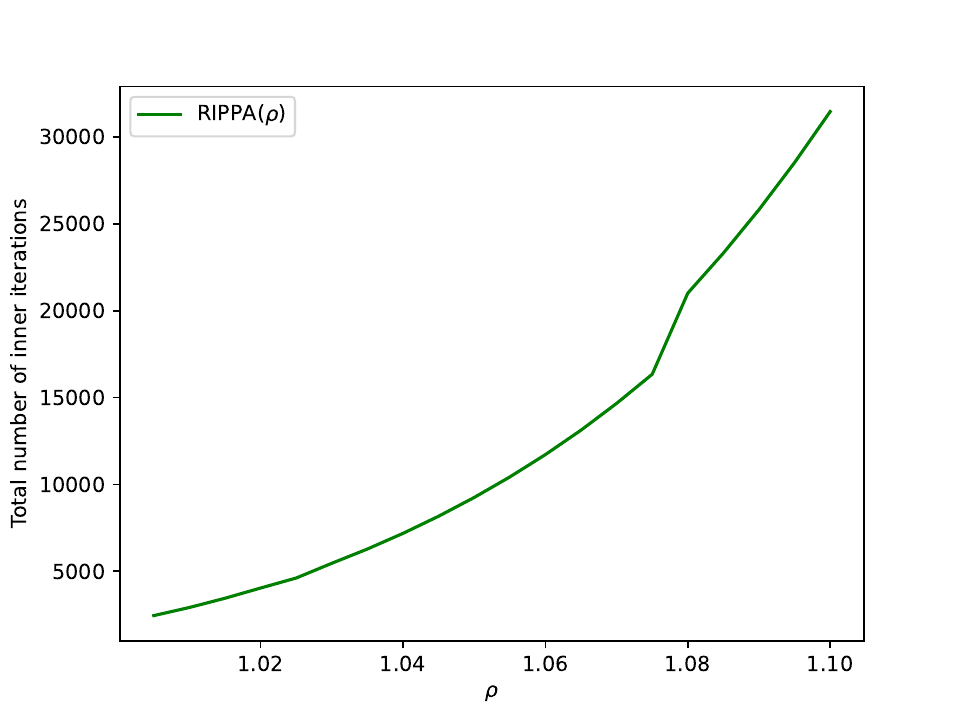}\label{fig:rippa_rho}}
	\subfigure[]{\includegraphics[width=0.31\textwidth]{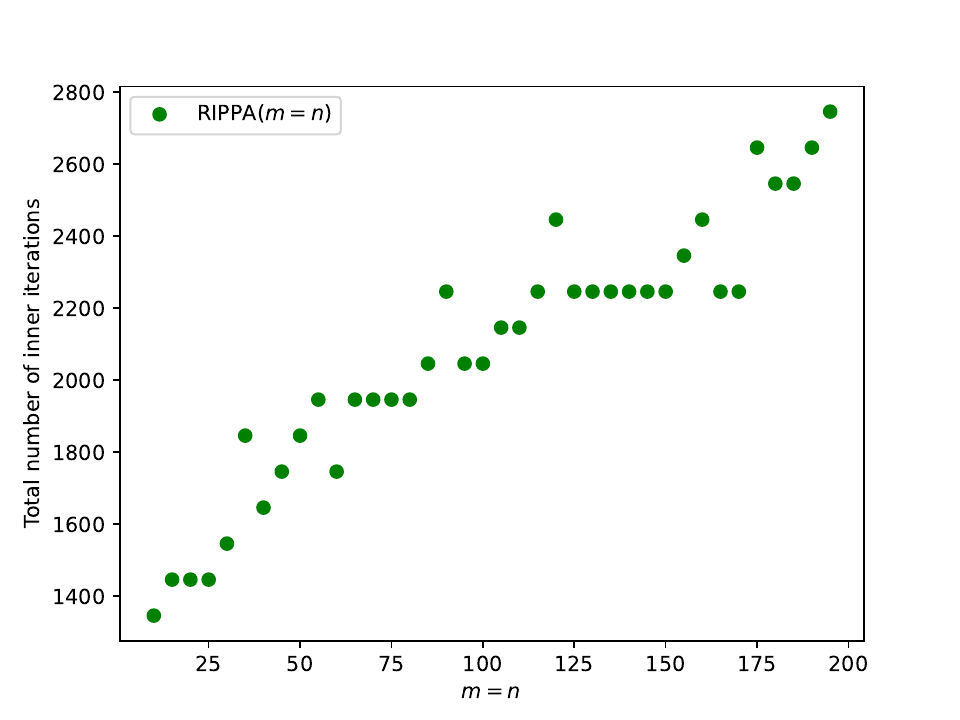}\label{fig:rippa_mn}}
	\subfigure[]{\includegraphics[width=0.31\textwidth]{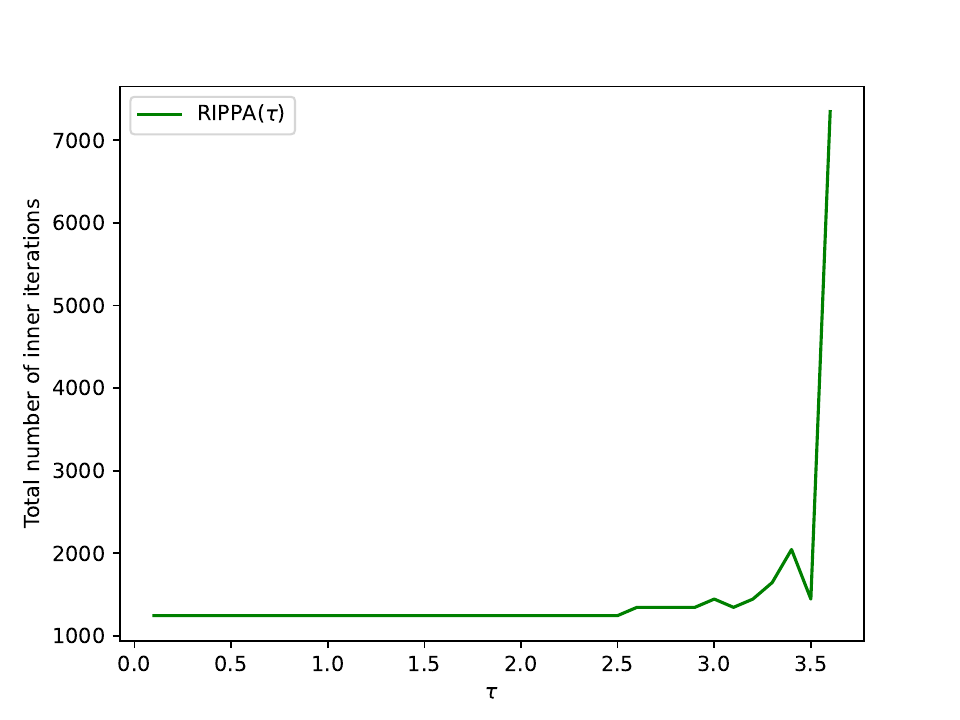}\label{fig:rippa_tau}}
  \caption{Total number of inner iterations needed for various parametrizations.
	(a) Varying $\rho$.
	(b) Varied problem dimensions where we set $m=n$.
	(c) Varying $\tau$}
\label{fig:l1ls}
\end{figure}

The first experiment from Figure \ref{fig:rippa_rho} investigates the effect of the $\rho$ parameter on the unconstrained $\ell_1$-LS formulation ($\tau=\infty$) on a small $50\times 20$ problem.
In our experiment we start with $\mu=0.1$, with 9 epochs and vary $\rho$ from 1.005 to 1.1 in 0.005 steps sizes.
In Figure~\ref{fig:rippa_mn}, we repeat the same experiment with fixed $\rho=1.005$ now, but with varied problem dimensions starting from 10 up to 200 in increments of 5 where we set both dimensions equal ($m=n$).
Finally, in Figure~\ref{fig:rippa_tau}, we study the effect of the problem specific parameter on
the total number of iterations. Although dim effects are noticed in the beginning, we can see a sudden burst past $\tau=3.4$. Please note that this is specific to $\ell_1$-LS and not to IPPA in general as we will see in the following section.

\subsection{Graph Support Vector Machines}

Graph SVM adds a graph-guided lasso regularization to the standard SVM averaged hinge-loss objective
and extends the $\ell_1$-SVM formulation through the factorization $Mx$ where $M$ is the weighted graph adjacency matrix, i.e.
\begin{align}
	\min\limits_{x \in \rset^n} \; & \; \frac{1}{m}\sum_{i=1}^m \max \{0, 1 - y_i a_i^T x \} + \tau\norm{Mx}_1
	\label{graphsvm}
\end{align}
where $a_i \in \rset^n$, $y_i \in \{\pm1\}$ are the $i-$th data point and its label, respectively.
When $M=I_n$ the Regularized Sparse $\ell_1$-SVM formulation is recovered.

\begin{figure}
	\centering

    Synthetic \\
	\subfigure[]{\includegraphics[width=0.49\textwidth]{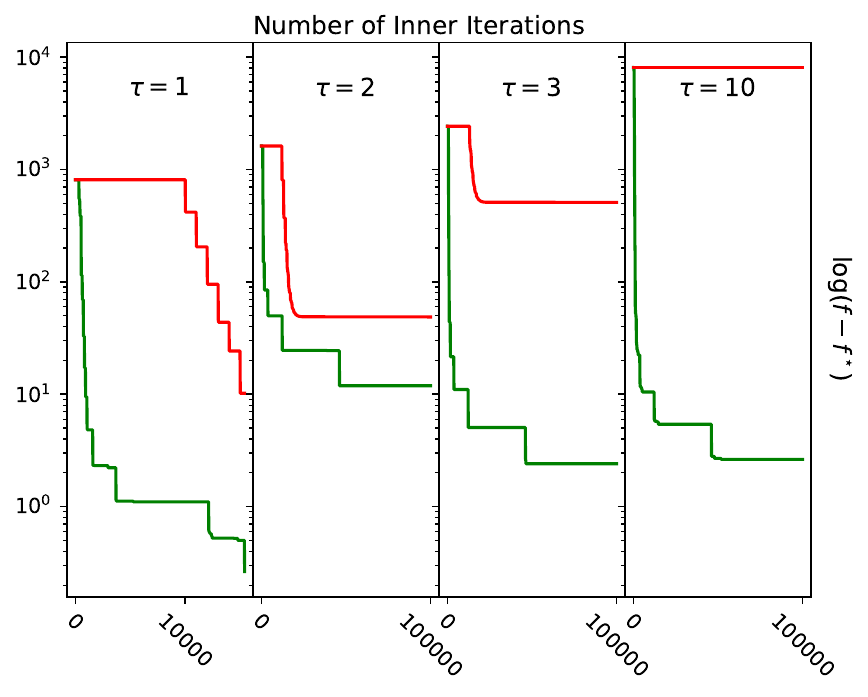}}
	\subfigure[]{\raisebox{3mm}{\includegraphics[width=0.49\textwidth]{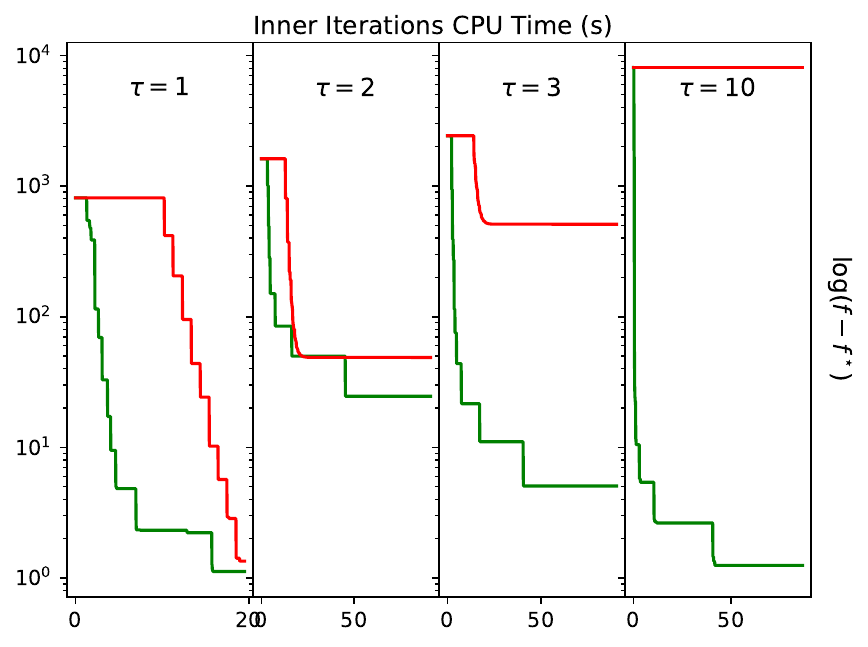}}}

    20newsgroups \\
	\subfigure[]{\includegraphics[width=0.49\textwidth]{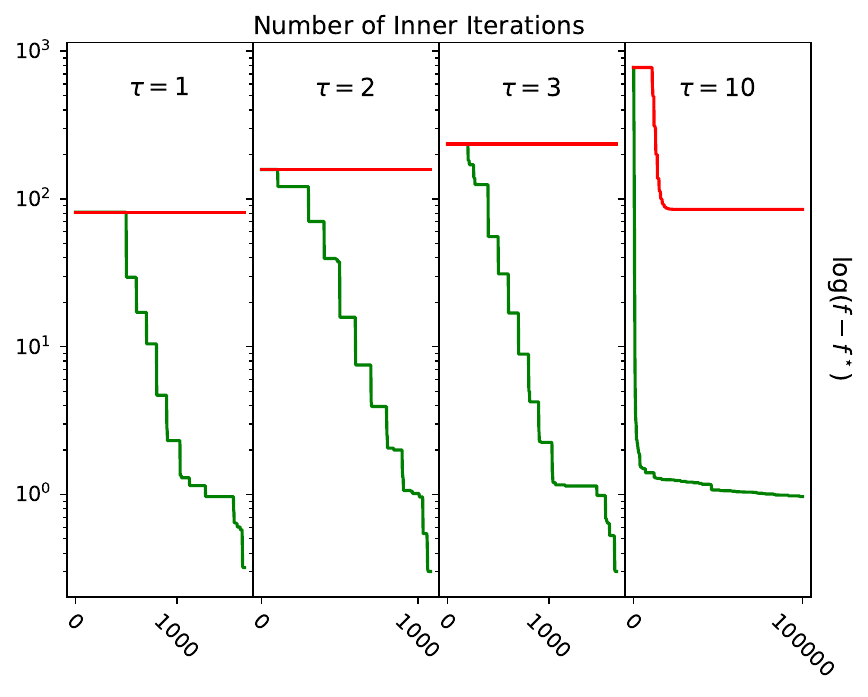}}
	\subfigure[]{\raisebox{3mm}{\includegraphics[width=0.49\textwidth]{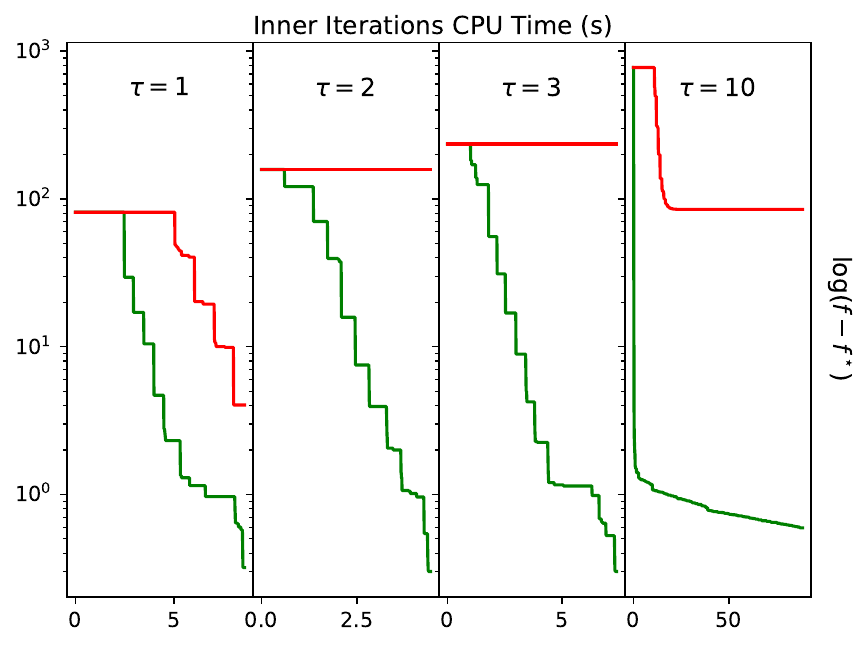}}}

    $\ell_1$-SVM ($M = I_n$) \\
    \subfigure[]{\includegraphics[width=0.49\textwidth]{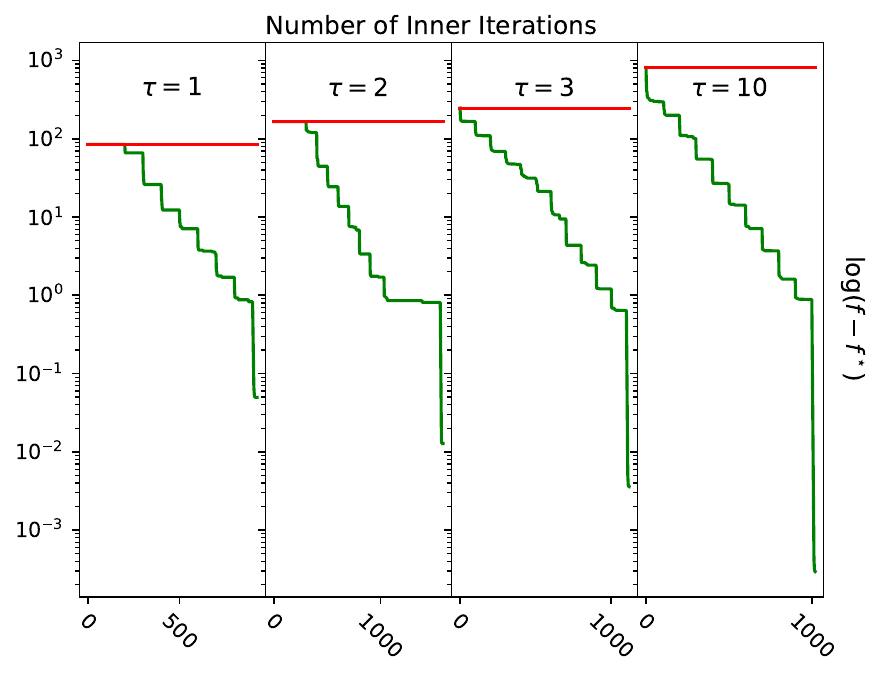}}
    \subfigure[]{\raisebox{2mm}{\includegraphics[width=0.49\textwidth]{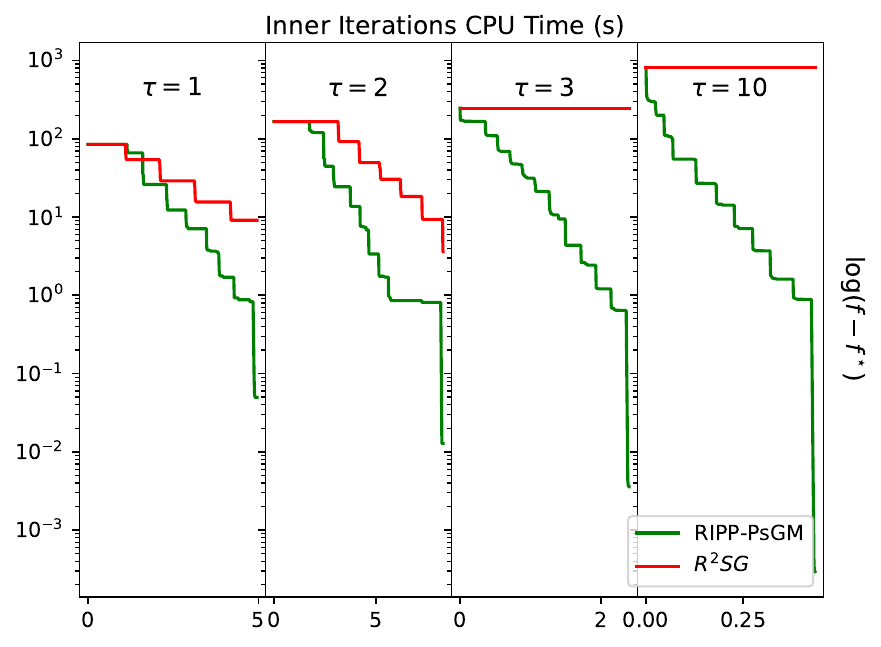}}}

  \caption{GraphSVM experiments on synthetic data, first row, on the 20newsgroups data-set, second row, and with $M=I_n$ on the third row.
	(a), (c), (e): Objective error evolution across inner iterations.
	(b), (d), (f): Objective error evolution across inner iterations measured in CPU time.
    The error is displayed in logarithmic scale.}
\label{fig:graphsvm}
\end{figure}

In Figure~\ref{fig:graphsvm} we present multiple experiments on different data and across different $\tau$ parametrizations. Figure \ref{fig:graphsvm} (a) and (b) compares RIPP-PsGM with R\textsuperscript{2}SG from \cite{Yan:18} on synthetic random data $\{C,X,M\}$ from the standard normal distribution.
The same initialization and starting point $x_0$ was used for all methods.
We use $m=100$ measurements of $n=512$ samples $x$
with initial parameters $\mu=0.1$, $\rho=1.0005$ and $\tau=1$ which we execute for 15 epochs. 

We repeat the experiment in Figure \ref{fig:graphsvm} (c) and (d), but this time on real-data
from the 20newsgroup data-set\footnote{https://cs.nyu.edu/~roweis/data.html}
following the experiment from \cite{Yan:18},
with parameters $\mu=0.1$, $\rho=1.005$, and $\tau=3$.
Here we find a similar behaviour for both methods as in the synthetic case. In Figure \ref{fig:graphsvm} (e) and (f),
we repeat the experiment by setting $M=I_m$ in \eqref{graphsvm} thus recovering the regularized $\ell_1$-SVM formulation.
We notice here that RIPP-PsGM maintains its position ahead of R\textsuperscript{2}SG and that the error drop is similar to the 20newsgroup experiment.

\begin{figure}[t]
	\centering
	\subfigure[]{\includegraphics[width=0.31\textwidth]{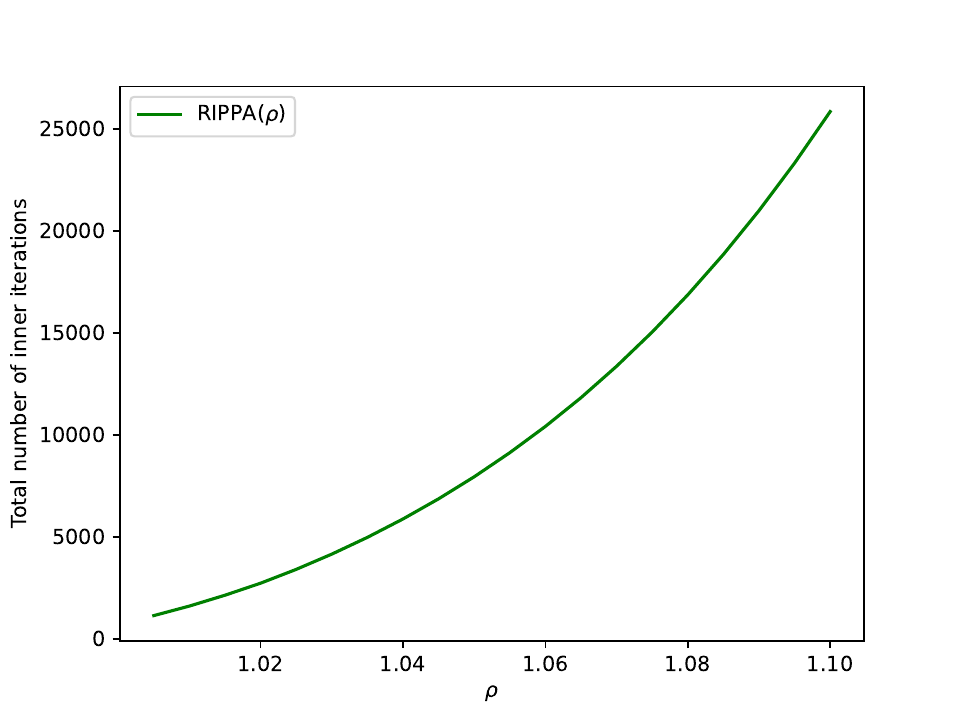}\label{fig:rippa_graphsvm_rho}}
	\subfigure[]{\includegraphics[width=0.31\textwidth]{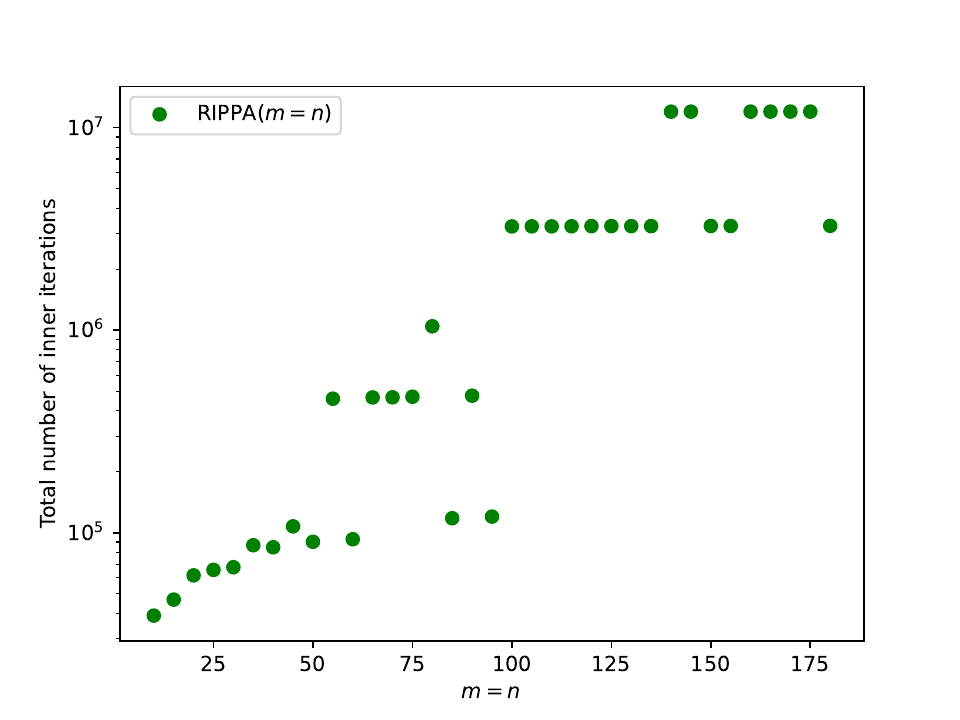}\label{fig:rippa_graphsvm_mn}}
	\subfigure[]{\includegraphics[width=0.31\textwidth]{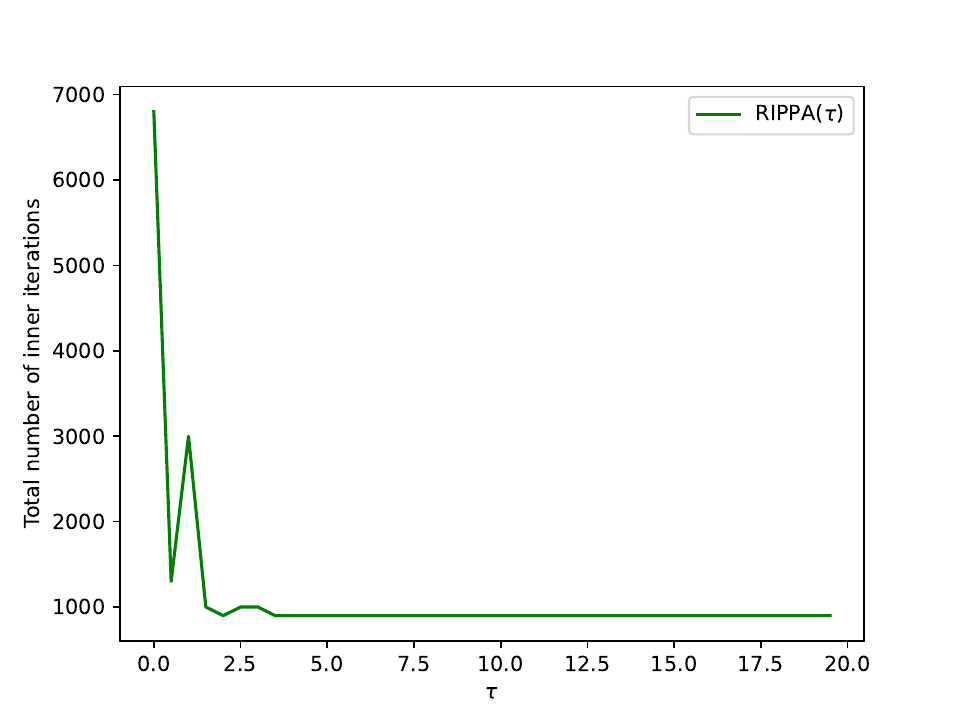}\label{fig:rippa_graphsvm_tau}}
	\caption{GraphSVM: Total number of inner iterations needed for various parametrizations.
		(a) Varying $\rho$.
		(b) Varied problem dimensions where we set $m=n$.
		(c) Varied $\tau$.}
	\label{fig:graphsvm_params}
\end{figure}

Furthermore, Figure~\ref{fig:graphsvm_params} rehashes the experiments from the $\ell_1$-LS Section in order to study the effect of the data dimensions and of the problem parameters on the number of total number of required inner iterations.
The results for $\rho$ and the data dimensions are as expected: as they grow they almost linearly increase the iteration numbers.
For the GraphSVM specific parameter $\tau$,
we find the results are opposite to that of $\ell_1$-LS;
it is harder to solve the problem when $\tau$ is small.

\begin{figure}[ht]
	\centering
	\subfigure[]{\includegraphics[width=0.49\textwidth]{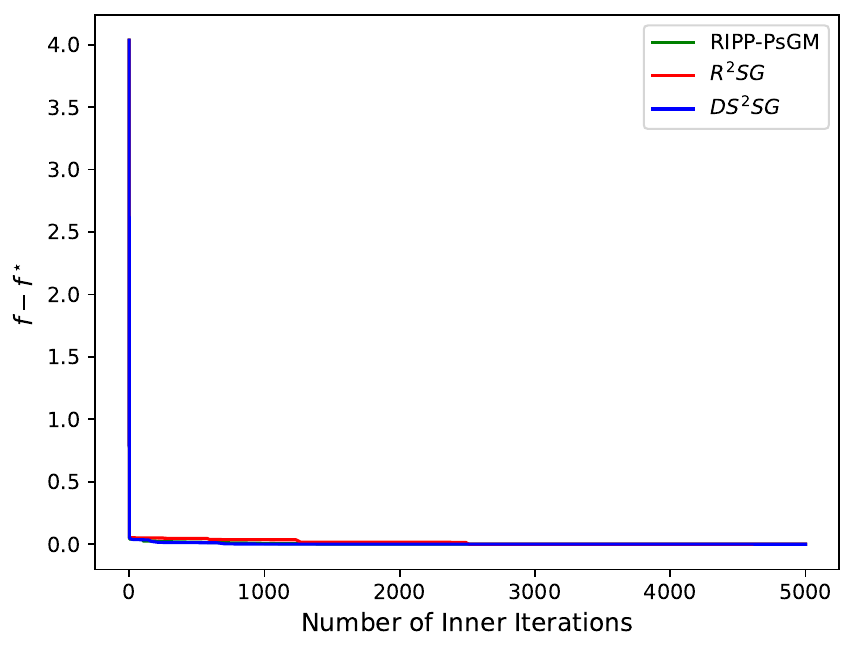}}
    \subfigure[]{\raisebox{5mm}{\includegraphics[width=0.5\textwidth]{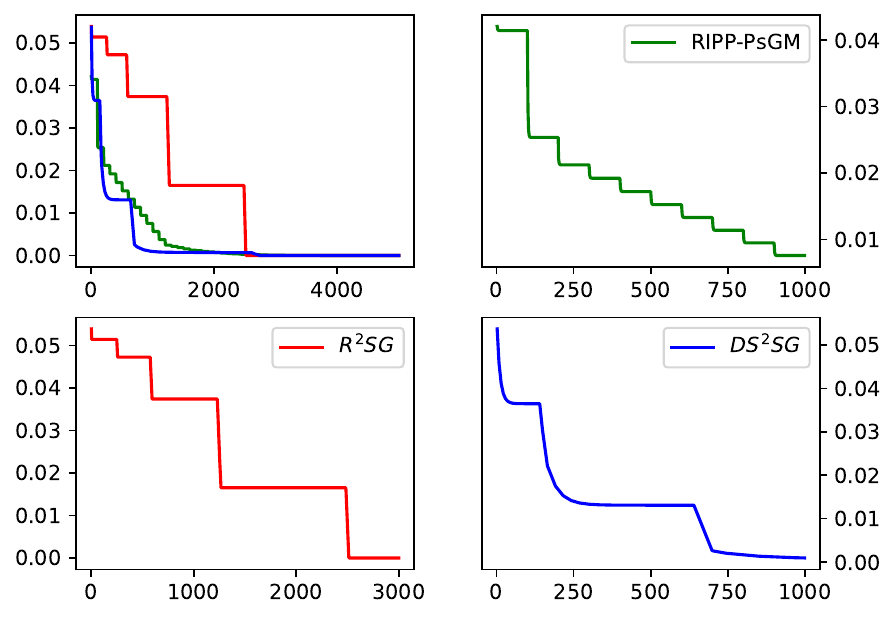}}}
  \caption{Sparse $\ell_1$-SVM: (a) Objective error evolution across inner iterations.
	(b) Zoomed objective error evolution.}
\label{fig:l1svm}
\end{figure}

\noindent In order to compare with DS\textsuperscript{2}SG~\cite{Joh:20} and R\textsuperscript{2}SG, we follow the Sparse SVM experiment from \cite{Joh:20} which requests $M=I_m$ and removes the regularization $\tau\norm{x}_1$ from \eqref{graphsvm} and instead
adds it as a constraint $\{x: \norm{x}_1 \le \tau\}$.
We used the parameters $\mu=0.001$, $\rho=1.0005$, and $\tau=0.4$ and plotted the results in Figure~\ref{fig:l1svm}.
Please note that the starting point is the same, with a quick initial drop for all three methods as can be seen in Figure~\ref{fig:l1svm} (a).
Execution times were almost identical and took around 0.15s.
To investigate the differences in convergence behaviour,
we zoom in after the first few iterations in the 4 panels of Figure~\ref{fig:l1svm} (b).
In the first panel we show the curves for all methods together and in the other three we see the individual curves for each.
Our experiments showed that this is the general behaviour for $\tau>0.4$:
DS\textsuperscript{2}SG has a slightly sharper drop,
with RIPP-PsGM following in closely with a staircase behaviour
while R\textsuperscript{2}SG takes a few more iterations to reach convergence.
For smaller values of $\tau$ we found that RIPP-PsGM reaches the solution in just 1--5 iterations within
$10^{-9}$ precision,
while the others lag behind for a few hundred iterations.

\subsection{Matrix Completion for Movie Recommendation}

In this last section,
the problem of matrix completion is applied
to the standard movie recommendation challenge which recovers a full user-rating  matrix $X$ from the partial observations $Y$
corresponding to the $N$ known user-movie ratings pairs.
\begin{align*}
	\min\limits_{X \in \rset^{m\times n}} \; & \; \frac{1}{N}\sum_{(i,j)\in\Sigma}^N |X_{ij} - Y_{ij}| + \tau\norm{X}_\star
\end{align*}
where $\Sigma$ is the set of user-movie pairs with $N = |\Sigma|$. Solving this will complete matrix X based on the known sparse matrix Y while maintaining a low rank.

\begin{figure}[t]
	\centering
	\subfigure[]{\includegraphics[width=0.49\textwidth]{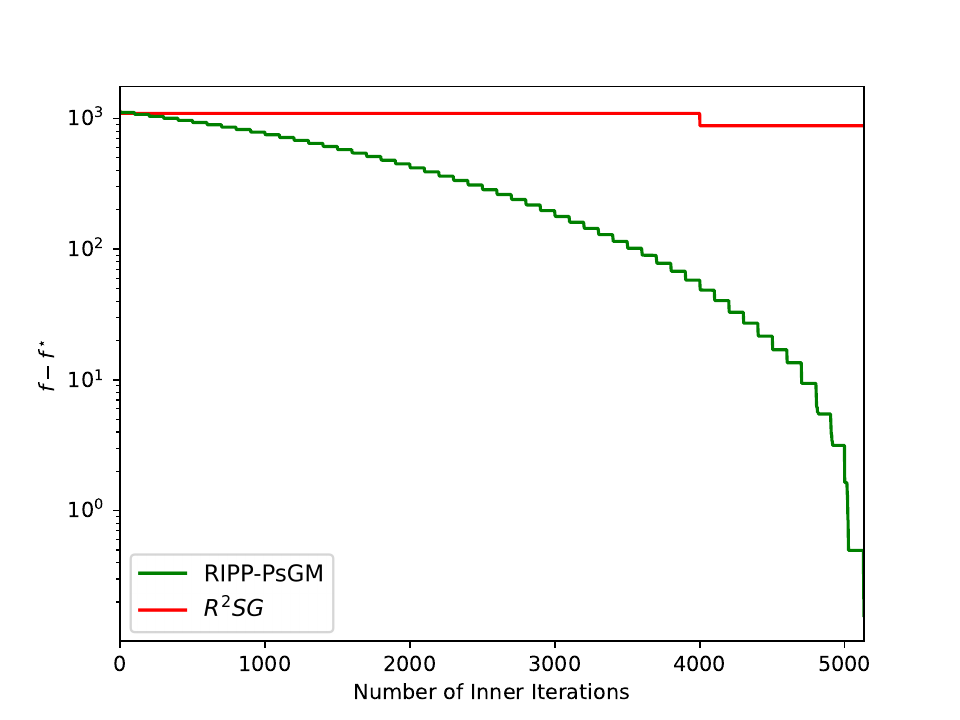}}
	\subfigure[]{\includegraphics[width=0.49\textwidth]{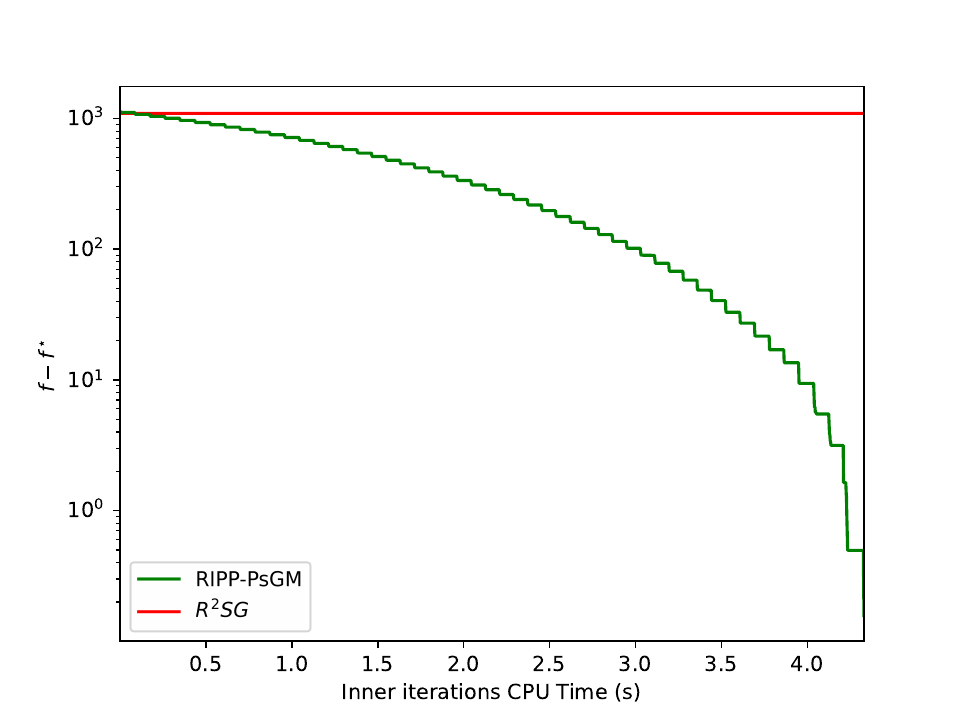}}
  \caption{Matrix completion: (a) Objective error evolution across inner iterations.
	(b) Objective error evolution across inner iterations measured in CPU time.}
\label{fig:mc}
\end{figure}

In Figure~\ref{fig:mc} we reproduce the experiment from \cite{Yan:18}
with parameters $\mu=0.1$, $\rho=1.005$, $\tau=3$
on a synthetic database with 50 movies and 20 users
filled with 250 i.i.d. randomly chosen ratings from 1 to 5.
We let a few more R\textsuperscript{2}SG iterations execute to show that the slight progress that is made.

\bibliographystyle{spmpsci}
\bibliography{inexact}

\begin{thebibliography}{10}
\providecommand{\url}[1]{{#1}}
\providecommand{\urlprefix}{URL }
\expandafter\ifx\csname urlstyle\endcsname\relax
  \providecommand{\doi}[1]{DOI~\discretionary{}{}{}#1}\else
  \providecommand{\doi}{DOI~\discretionary{}{}{}\begingroup
  \urlstyle{rm}\Url}\fi

\bibitem{Ant:94}
Antipin, A.: On finite convergence of processes to a sharp minimum and to a
  smooth minimum with a sharp derivative.
\newblock Differential Equations \textbf{30}(11), 1703--1713 (1994)

\bibitem{BauDao:15}
Bauschke, H.H., Dao, M.N., Noll, D., Phan, H.M.: Proximal point algorithm,
  douglas-rachford algorithm and alternating projections: a case study.
\newblock Journal of Convex Analysis \textbf{23}(1), 237--261 (2015)

\bibitem{BecTeb:09fista}
Beck, A., Teboulle, M.: A fast iterative shrinkage-thresholding algorithm for
  linear inverse problems.
\newblock SIAM journal on imaging sciences \textbf{2}(1), 183--202 (2009)

\bibitem{Ber:82constrained}
Bertsekas, D.P.: Constrained optimization and Lagrange multiplier methods.
\newblock Athena Scientific (1982)

\bibitem{Ber:89parallel}
Bertsekas, D.P.: Parallel and distributed computation: numerical methods,
  vol.~23.
\newblock Prentice hall Englewood Cliffs, NJ (1989)

\bibitem{BolNgu:17}
Bolte, J., Nguyen, T.P., Peypouquet, J., Suter, B.W.: From error bounds to the
  complexity of first-order descent methods for convex functions.
\newblock Mathematical Programming \textbf{165}(2), 471--507 (2017)

\bibitem{BurFer:93}
Burke, J.V., Ferris, M.C.: Weak sharp minima in mathematical programming.
\newblock SIAM Journal on Control and Optimization \textbf{31}(5), 1340--1359
  (1993)

\bibitem{Dav:18}
Davis, D., Drusvyatskiy, D., MacPhee, K.J., Paquette, C.: Subgradient methods
  for sharp weakly convex functions.
\newblock Journal of Optimization Theory and Applications \textbf{179}(3),
  962--982 (2018)

\bibitem{Fer:91}
Ferris, M.C.: Finite termination of the proximal point algorithm.
\newblock Mathematical Programming \textbf{50}(1), 359--366 (1991)

\bibitem{Fre:18}
Freund, R.M., Lu, H.: New computational guarantees for solving convex
  optimization problems with first order methods, via a function growth
  condition measure.
\newblock Mathematical Programming \textbf{170}(2), 445--477 (2018)

\bibitem{Gil:12}
Gilpin, A., Pena, J., Sandholm, T.: First-order algorithm with $\mathcal{O}(\ln
  (1 /\epsilon))$ convergence for $\epsilon$-equilibrium in two-person zero-sum
  games.
\newblock Mathematical programming \textbf{133}(1), 279--298 (2012)

\bibitem{Gul:91}
G{\"u}ler, O.: On the convergence of the proximal point algorithm for convex
  minimization.
\newblock SIAM journal on control and optimization \textbf{29}(2), 403--419
  (1991)

\bibitem{Gul:92}
G{\"u}ler, O.: New proximal point algorithms for convex minimization.
\newblock SIAM Journal on Optimization \textbf{2}(4), 649--664 (1992)

\bibitem{Hu:16}
Hu, Y., Li, C., Yang, X.: On convergence rates of linearized proximal
  algorithms for convex composite optimization with applications.
\newblock SIAM Journal on Optimization \textbf{26}(2), 1207--1235 (2016)

\bibitem{Hum:05}
Humes, C., Silva, P.J.: Inexact proximal point algorithms and descent methods
  in optimization.
\newblock Optimization and Engineering \textbf{6}(2), 257--271 (2005)

\bibitem{Joh:20}
Johnstone, P.R., Moulin, P.: Faster subgradient methods for functions with
  h{\"o}lderian growth.
\newblock Mathematical Programming \textbf{180}(1), 417--450 (2020)

\bibitem{JudNes:14uniform}
Juditsky, A., Nesterov, Y.: Deterministic and stochastic primal-dual
  subgradient algorithms for uniformly convex minimization.
\newblock Stochastic Systems \textbf{4}(1), 44--80 (2014)

\bibitem{Kor:76}
Kort, B.W., Bertsekas, D.P.: Combined primal--dual and penalty methods for
  convex programming.
\newblock SIAM Journal on Control and Optimization \textbf{14}(2), 268--294
  (1976)

\bibitem{Li:12}
Li, G., Mordukhovich, B.S.: Holder metric subregularity with applications to
  proximal point method.
\newblock SIAM Journal on Optimization \textbf{22}(4), 1655--1684 (2012)

\bibitem{Lin:15}
Lin, H., Mairal, J., Harchaoui, Z.: A universal catalyst for first-order
  optimization.
\newblock Advances in neural information processing systems \textbf{28} (2015)

\bibitem{Lin:18}
Lin, H., Mairal, J., Harchaoui, Z.: Catalyst acceleration for first-order
  convex optimization: from theory to practice.
\newblock Journal of Machine Learning Research \textbf{18}(1), 7854--7907
  (2018)

\bibitem{Lu:20}
Lu, M., Qu, Z.: An adaptive proximal point algorithm framework and application
  to large-scale optimization.
\newblock arXiv preprint arXiv:2008.08784  (2020)

\bibitem{Luo:93}
Luo, Z.Q., Tseng, P.: Error bounds and convergence analysis of feasible descent
  methods: a general approach.
\newblock Annals of Operations Research \textbf{46}(1), 157--178 (1993)

\bibitem{Mai:19}
Mairal, J.: Cyanure: An open-source toolbox for empirical risk minimization for
  python, c++, and soon more.
\newblock arXiv preprint arXiv:1912.08165  (2019)

\bibitem{Mon:13}
Monteiro, R.D., Svaiter, B.F.: An accelerated hybrid proximal extragradient
  method for convex optimization and its implications to second-order methods.
\newblock SIAM Journal on Optimization \textbf{23}(2), 1092--1125 (2013)

\bibitem{Nec:19}
Necoara, I., Nesterov, Y., Glineur, F.: Linear convergence of first order
  methods for non-strongly convex optimization.
\newblock Mathematical Programming \textbf{175}(1), 69--107 (2019)

\bibitem{Ned:10}
Nedi{\'c}, A., Bertsekas, D.P.: The effect of deterministic noise in
  subgradient methods.
\newblock Mathematical programming \textbf{125}(1), 75--99 (2010)

\bibitem{Nem:04}
Nemirovski, A.: Prox-method with rate of convergence o (1/t) for variational
  inequalities with lipschitz continuous monotone operators and smooth
  convex-concave saddle point problems.
\newblock SIAM Journal on Optimization \textbf{15}(1), 229--251 (2004)

\bibitem{NemNes:85}
Nemirovskii, A., Nesterov, Y.: Optimal methods of smooth convex minimization.
\newblock USSR Computational Mathematics and Mathematical Physics
  \textbf{25}(2), 21--30 (1985).
\newblock \doi{https://doi.org/10.1016/0041-5553(85)90100-4}.
\newblock
  \urlprefix\url{https://www.sciencedirect.com/science/article/pii/0041555385901004}

\bibitem{Nes:12Optima}
Nesterov, Y.: How to make the gradients small.
\newblock Optima \textbf{88} (2012)

\bibitem{Nes:21}
Nesterov, Y.: Inexact accelerated high-order proximal-point methods.
\newblock Mathematical Programming pp. 1--26 (2021)

\bibitem{Nes:21b}
Nesterov, Y.: Inexact high-order proximal-point methods with auxiliary search
  procedure.
\newblock SIAM Journal on Optimization \textbf{31}(4), 2807--2828 (2021)

\bibitem{PatIroLett:22}
Patrascu, A., Irofti, P.: On finite termination of an inexact proximal point
  algorithm.
\newblock Applied Mathematics Letters \textbf{134}, 108348 (2022).
\newblock \doi{https://doi.org/10.1016/j.aml.2022.108348}

\bibitem{Pol:67SM1}
Polyak, B.: A general method of solving extremal problems.
\newblock Math. Doklady \textbf{8}, 593--597 (1967)

\bibitem{Pol:69SM2}
Polyak, B.: Minimization of unsmooth functionals.
\newblock U.S.S.R. Computational Mathematics and Mathematical Physics
  \textbf{9}, 509--521 (1969)

\bibitem{Pol:78}
Polyak, B.: Nonlinear programming methods in the presence of noise.
\newblock Mathematical Programming \textbf{14}, 87--97 (1978)

\bibitem{Pol:87book}
Polyak, B.: Introduction to optimization.
\newblock Optimization Software, Inc., Publications Division, New York  (1987)

\bibitem{Ren:14}
Renegar, J.: Efficient first-order methods for linear programming and
  semidefinite programming.
\newblock arXiv preprint arXiv:1409.5832  (2014)

\bibitem{Roc:76augmented}
Rockafellar, R.T.: Augmented lagrangians and applications of the proximal point
  algorithm in convex programming.
\newblock Mathematics of operations research \textbf{1}(2), 97--116 (1976)

\bibitem{Roc:76}
Rockafellar, R.T.: Monotone operators and the proximal point algorithm.
\newblock SIAM journal on control and optimization \textbf{14}(5), 877--898
  (1976)

\bibitem{Rou:20}
Roulet, V., d'Aspremont, A.: Sharpness, restart, and acceleration.
\newblock SIAM Journal on Optimization \textbf{30}(1), 262--289 (2020)

\bibitem{Sal:12}
Salzo, S., Villa, S.: Inexact and accelerated proximal point algorithms.
\newblock Journal of Convex analysis \textbf{19}(4), 1167--1192 (2012)

\bibitem{Sho:62}
Shor, N.: An application of the method of gradient descent to the solution of
  the network transportation problem.
\newblock Materialy Naucnovo Seminara po Teoret i Priklad. Voprosam Kibernet. i
  Issted. Operacii, Nuc- nyi Sov. po Kibernet, Akad. Nauk Ukrain. SSSR
  \textbf{1}, 9--17 (1962)

\bibitem{Sho:64}
Shor, N.: On the structure of algorithms for numerical solution of problems of
  optimal planning and design.
\newblock Diss. Doctor Philos, Kiev  (1964)

\bibitem{Shulgin:21}
Shulgin, E., Gasnikov, A., Matyukhin, V.: Adaptive catalyst for smooth convex
  optimization.
\newblock In: Optimization and Applications: 12th International Conference,
  OPTIMA 2021, Petrovac, Montenegro, September 27--October 1, 2021,
  Proceedings, vol. 13078, p.~20. Springer Nature (2021)

\bibitem{Sol:99b}
Solodov, M.V., Svaiter, B.F.: A hybrid approximate extragradient--proximal
  point algorithm using the enlargement of a maximal monotone operator.
\newblock Set-Valued Analysis \textbf{7}(4), 323--345 (1999)

\bibitem{Sol:99}
Solodov, M.V., Svaiter, B.F.: A hybrid projection-proximal point algorithm.
\newblock Journal of convex analysis \textbf{6}(1), 59--70 (1999)

\bibitem{Sol:00}
Solodov, M.V., Svaiter, B.F.: Error bounds for proximal point subproblems and
  associated inexact proximal point algorithms.
\newblock Mathematical programming \textbf{88}(2), 371--389 (2000)

\bibitem{Sol:01}
Solodov, M.V., Svaiter, B.F.: A unified framework for some inexact proximal
  point algorithms.
\newblock Numerical functional analysis and optimization \textbf{22}(7-8),
  1013--1035 (2001)

\bibitem{Sol:01b}
Solodov, M.V., Svaiter, B.F.: A unified framework for some inexact proximal
  point algorithms.
\newblock Numerical functional analysis and optimization \textbf{22}(7-8),
  1013--1035 (2001)

\bibitem{Tom:11}
Tomioka, R., Suzuki, T., Sugiyama, M.: Super-linear convergence of dual
  augmented lagrangian algorithm for sparsity regularized estimation.
\newblock Journal of Machine Learning Research \textbf{12}(5) (2011)

\bibitem{XiaLin:14}
Xiao, L., Zhang, T.: A proximal stochastic gradient method with progressive
  variance reduction.
\newblock SIAM Journal on Optimization \textbf{24}(4), 2057--2075 (2014)

\bibitem{Yan:18}
Yang, T., Lin, Q.: Rsg: Beating subgradient method without smoothness and
  strong convexity.
\newblock The Journal of Machine Learning Research \textbf{19}(1), 236--268
  (2018)

\bibitem{YanEla:16}
Yankelevsky, Y., Elad, M.: Dual graph regularized dictionary learning.
\newblock IEEE Transactions on Signal and Information Processing over Networks
  \textbf{2}(4), 611--624 (2016)

\end{thebibliography}


\section{Appendix}



\begin{proof}[Proof of Lemma \ref{lemma:deterministic_moreau_growth}]
	By using $\gamma$-HG, we get:
	\begin{align}
		F_{\mu}(x) - F^* 
		& = \min_z F(z) - F^* + \frac{1}{2\mu}\norm{z-x}^2 \nonumber \\
		& \ge \min_z \sigma_F\dist_{X^*}^{\gamma}(z)  + \frac{1}{2\mu}\norm{z-x}^2 \nonumber \\
		& = \min_{z,y \in X^*} \sigma_F\norm{z-y}^{\gamma}  + \frac{1}{2\mu}\norm{z-x}^2. \label{eq:eq:prelim_Moreau_growth}
	\end{align}
	The solution of \eqref{eq:eq:prelim_Moreau_growth} in $z$, denoted as $z(x)$  satisfies the following optimality condition: $\sigma_F\gamma \frac{z(x) - y}{\norm{z(x) - y}^{2 - \gamma}} + \frac{1}{\mu}\left(z(x) - x \right) = 0$, which simply implies that 
	\begin{align}\label{eq:solution_Holder_subproblem}
		z(x) = \frac{\norm{z-y}^{2-\gamma}}{\norm{z-y}^{2-\gamma} + \sigma_F \mu \gamma} x + \frac{\sigma_F \mu \gamma}{\norm{z-y}^{2-\gamma} + \sigma_F \mu \gamma} y.
	\end{align}
	$(i)$	For a function with sharp minima ($\gamma = 1$) it is easy to see that $(z(x) - y) \left[1 + \frac{\sigma_F \mu}{\norm{z(x) - y}} \right] = x - y$. By taking norm in both sides then: $\norm{z(x)- y} = \max\{0, \norm{x-y} - \sigma_F \mu\}$ and \eqref{eq:solution_Holder_subproblem} becomes:
	$$ z(x) = \begin{cases} y + \left( 1 - \frac{\sigma_F\mu}{\norm{x-y}} \right) (x-y), & \norm{x-y} > \sigma_F \mu\\ 
		y, & \norm{x-y} \le \sigma_F \mu \end{cases}$$
	By replacing this form of $z(x)$ into \eqref{eq:eq:prelim_Moreau_growth} we obtain our first result:
	\begin{align*}
		F_{\mu}(x) - F^* 
		& \ge \min_{y \in X^*} 
		\begin{cases}\sigma_F \norm{x-y}  - \frac{\mu\sigma-F^2}{2}, & \norm{x-y} > \sigma_F\mu \\ 
			\frac{1}{2\mu}\norm{y-x}^2, & \norm{x-y} \le \sigma_F\mu
		\end{cases} \\
		& \ge 
		\begin{cases}\sigma_F \dist_{X^*}(x)  - \frac{\mu\sigma^2_F}{2}, & \dist_{X^*}(x) > \sigma_F \mu \\ 
			\frac{1}{2\mu}\dist_{X^*}(x)^2, & \dist_{X^*}(x) \le \sigma_F \mu
		\end{cases}. 
	\end{align*}

	\noindent $(ii)$ For quadratic growth, \eqref{eq:solution_Holder_subproblem} reduces to $z(x) = \frac{1}{1 + 2\sigma_F \mu} x + \frac{2\sigma_F \mu }{1 + 2\sigma_F \mu} y$ and by \eqref{eq:eq:prelim_Moreau_growth} it leads to:
	\begin{align*}
		F_{\mu}(x) - F^* 
		& \ge \min_{y \in X^*}  \frac{\sigma_F}{1 + 2\sigma_F \mu}\norm{y-x}^2 = \frac{\sigma_F}{1 + 2\sigma_F \mu}\dist_{X^*}^2(x).
	\end{align*}
	
	\noindent  $(iii)$ Lastly, from \eqref{eq:solution_Holder_subproblem} we see that $z(x)$ lies on the segment $[x,y]$, i.e.  $z(x) = \lambda x + (1-\lambda)y$ for certain $\lambda \in [0, 1]$. Using this argument into \eqref{eq:eq:prelim_Moreau_growth} we equivalently have:
	\begin{align*}
		F_{\mu}(x) - F^* & \ge \min_{y \in X^*, \lambda \in [0,1], z = \lambda x + (1-\lambda)y} \sigma_F\norm{z-y}^{\gamma}  + \frac{1}{2\mu}\norm{z-x}^2 \\
		& = \min_{y \in X^*, \lambda \in [0,1]} \sigma_F\lambda^{\gamma}\norm{x-y}^{\gamma}  + \frac{(1-\lambda)^2}{2\mu}\norm{x-y}^2 \\
		&= \min_{\lambda \in [0,1]} \sigma_F\lambda^{\gamma}\dist^{\gamma}_{X^*}(x)  + \frac{(1-\lambda)^2}{2\mu}\dist^{2}_{X^*}(x) \\
		& \ge \min \left\{   \sigma_F\dist^{\gamma}_{X^*}(x), \frac{1}{2\mu}\dist^{2}_{X^*}(x) \right\} \min_{\lambda \in [0,1]} \lambda^{\gamma} + (1-\lambda)^2.
	\end{align*}
\end{proof}


\begin{proof}[Proof of Lemma \ref{lemma:morenvgrad_bound_relating_with_distopt}]
	Assume $\norm{\nabla_{\delta} F_{\mu}(x)}  \le \epsilon$. Observe that on one hand, by the triangle inequality that: $\norm{\nabla F_{\mu}(x)} \le \norm{\nabla_{\delta} F_{\mu}(x)} + \frac{\delta}{\mu} \le \epsilon + \frac{\delta}{\mu} =: \tilde{\epsilon}$. On the other hand,
	one can derive:
	\begin{align}
		\norm{\nabla F_{\mu}(x)} 
		\le \sqrt{\frac{2}{\mu}[F_{\mu}(x)-F^*]} 
		&  \le  
		\frac{1}{\mu}\dist_{X^*}(x). \label{upperbound_morenvgrad}
	\end{align}
	
	\noindent Now let $\gamma = 1$. Based on $\nabla F_{\mu}(x) \in \partial F(\prox_{\mu}^F(x))$, for nonoptimal $\prox_{\mu}^F(x)$ the bound \eqref{gamma1_Env_Holder_error_bound} guarantees $  \norm{\nabla F_{\mu}(x)} = \norm{F'(\prox_{\mu}^F(x))} \ge \sigma_F  $. Therefore, if $x$ would determine $\tilde{\epsilon} < \sigma_F$, and implicitly $\norm{\nabla F_{\mu}(x)} < \sigma_F$, the contradiction with the previous lower bound impose that $\prox_{\mu}^F(x) \in X^*$. Moreover, in this case, since $ \norm{x - \prox_{\mu}^F(x)} = \mu \norm{\nabla F_{\mu}(x)} \overset{\eqref{upperbound_morenvgrad}}{\le} \dist_{X^*}(x)$ then obviously $\prox_{\mu}^F(x) = \pi_{X^*}(x)$. On summary, a sufficiently small approximate gradient norm $\norm{\nabla_{\delta} F_{\mu}(x)} < \sigma_F - \frac{\delta}{\mu}$ also confirms a small distance to optimal set $\dist_{X^*}(x) \le \mu\tilde{\epsilon}$. 
	
	
	\noindent Let $\gamma > 1$. Then our assumption $\norm{\nabla F_{\mu}(x)} \le \tilde{\epsilon}$, \eqref{upperbound_morenvgrad} and \eqref{Holder_error_bound} implies the error bound:
	$\;\;	\dist_{X^*}(x) \le \Max{\left[ \frac{\tilde{\epsilon}}{\sigma_F \varphi(\gamma)} \right]^{\frac{1}{\gamma  - 1}}, \frac{2\mu \tilde{\epsilon}}{\varphi(\gamma)}}$,
	which confirms the last result.
\end{proof}


\begin{proof}[Proof of  Corrolary \ref{cor:complexity_sharp_minima}]
	From Theorem \ref{th:IPPconvergence}$(i)$ we obtain directly:
	\begin{align*}
		\dist_{X^*}(x^{k}) 
		\le  \max\left\{\dist_{X^*}(x^{0}) - k(\mu\sigma_F - \delta), \delta  \right\}, 
	\end{align*}
	which means that after at most $K$ iterations $x^k$ reaches $\dist_{X^*}(x^{K}) \le \delta < \mu\sigma_F$. Lastly, the same reasoning as in Lemma \ref{lemma:morenvgrad_bound_relating_with_distopt}, based on the relations \eqref{upperbound_morenvgrad} and \eqref{gamma1_Env_Holder_error_bound},  lead to $\prox_{\mu}^F(x^K) = \pi_{X^*}(x^K)$.
\end{proof}


\begin{proof}[Proof of  Corrolary \ref{cor:complexity_for_exact_case}]
	The proof for the first two estimates are immediately derived from Theorem \ref{th:IPPconvergence} $(i)$ and $(ii)$. For $\gamma \in (1,2)$, we considered $\alpha = \frac{\mu\varphi(\gamma)\sigma_F}{2}, \beta = \frac{1 + \sqrt{1-\sigma_F}}{2}$ into Corrolary \ref{corr:exact_complexity} and obtained an estimate for our exact case. To refine the complexity order, we majorized some constants by using: $(2\beta \mu\sigma_F)^{\frac{1}{2-\gamma}} \le (2 \mu\sigma_F)^{\frac{1}{2-\gamma}}$ and $1 - \sqrt{1-\varphi(\gamma)} < 1$. For $\gamma > 2$, we replace the same $\alpha$ as before and $\hat{\beta} = \Max{ \frac{1 + \sqrt{1-\sigma_F}}{2}, 1 - \frac{1}{\gamma - 1}   }$ into Corrolary \ref{corr:exact_complexity} to get the last estimate.
\end{proof}

\begin{proof}[Proof of Corrolary \ref{cor:complexity_for_inexact_case}]
	The proof for the first two estimates can be derived immediately from Theorem \ref{th:IPPconvergence} $(i)$ and $(ii)$. We provide details for the other two cases.
	
	For $\gamma \in (1,2)$ we use the same notations as in the proof of Theorem \ref{th:IPPconvergence} (given in the appendix). There, the key functions which decide the decrease rate of $\dist_{X^*}(x^k)$ are the nondecreasing function $h$ and accuracy $\hat{\delta}_k = \Max{\left( \frac{2\delta_k}{\mu\sigma_F \varphi(\gamma)} \right)^{\frac{1}{\gamma-1}},  \frac{2\delta_k}{1 - \sqrt{1-\varphi(\gamma)}}}$. First recall that 
	\begin{align}\label{varphi_bounds}
		\frac{1}{2} = \varphi(2) \le \varphi(\gamma) \le \varphi(1) = \frac{3}{4}.
	\end{align}
	which implies that for any $\delta \ge 0$
	\begin{align}\label{rel:low_bound_h}
		\frac{\delta}{2} \overset{\eqref{varphi_bounds}}{\le} \frac{1 + \sqrt{1-\varphi(\gamma)}}{2}\delta \le h(\delta).
	\end{align}
	Recalling that $\bar{\delta}_k = \Max{\hat{\delta}_{k}, h^{}(\bar{\delta}_{k-1})}$, then by Theorem \ref{th:IPPconvergence}$(iii)$ we have:
	\begin{align}\label{rel:IPPconvergence_reloaded}
		\dist_{X^*}(x^k) 
		&\le 
		\max\left\{h^{(k)}(\dist_{X^*}(x^0)), \bar{\delta}_k \right\}
	\end{align}	
	By taking $\delta_k = \frac{\delta_{k-1}}{2}$ then $ \hat{\delta}_{k} =
	\Max{\frac{1}{2^{\frac{1}{\gamma - 1}}}\left( \frac{2\delta_{k-1}}{\mu\sigma_F \varphi(\gamma)} \right)^{\frac{1}{\gamma-1}},  \frac{1}{2}\frac{2\delta_{k-1}}{1 - \sqrt{1-\varphi(\gamma)}}} \le \frac{\hat{\delta}_{k-1}}{2} \overset{\eqref{rel:low_bound_h}}{\le} h^{}(\hat{\delta}_{k-1})$. By this recurrence, the monotonicity of $h$ and $\hat{\delta}_0 = \bar{\delta}_0$, we derive: 
	\begin{align*}
		\bar{\delta}_k 
		& \le \Max{h(\hat{\delta}_{k-1}), h^{}(\bar{\delta}_{k-1})} = h^{}(\bar{\delta}_{k-1}) \\
		& = h^{(k)}(\hat{\delta}_{0}). 
	\end{align*}
	Finally this key bound enters into \eqref{rel:IPPconvergence_reloaded} and we get:
	\begin{align*}
		\dist_{X^*}(x^k) 
		& \le \Max{h^{(k)}\left(\dist_{X^*}(x^0)\right), h^{(k)}\left(\hat{\delta}_{0}\right)}\\ 
		& \le h^{(k)}\left(\Max{\dist_{X^*}(x^0), \hat{\delta}_{0}}\right), 
	\end{align*}
	where for the last equality we used the fact that, since $h$ is nondecreasing, $h^{(k)}$ is monotonically nondecreasing.
	Finally, by applying Theorem \ref{th:central_recurrence} we get our result.
	Now let $\gamma >2$. By redefining $h$ as in Theorem \ref{th:IPPconvergence}, observe that
	\begin{align}\label{rel:low_bound_h_hat}
		\delta \left( 1 - \frac{1}{\gamma - 1}\right)\le h(\delta).
	\end{align}
	Take $\delta_k = \left( 1 - \frac{1}{\gamma - 1}\right)^{\gamma - 1} \delta_{k-1}$ then 
	\begin{align*}
		\hat{\delta}_{k} & =
		\Max{\left( 1 - \frac{1}{\gamma - 1}\right)  \left( \frac{2\delta_{k-1}}{\mu\sigma_F \varphi(\gamma)} \right)^{\frac{1}{\gamma-1}},  \left( 1 - \frac{1}{\gamma - 1}\right)^{\gamma - 1} \frac{2\delta_{k-1}}{1 - \sqrt{1-\varphi(\gamma)}}} \\
		& \le \left( 1 - \frac{1}{\gamma - 1}\right) \hat{\delta}_{k-1} \overset{\eqref{rel:low_bound_h_hat}}{\le} h^{}(\hat{\delta}_{k-1}). 
	\end{align*}
	We have shown in the proof of Theorem \ref{th:IPPconvergence} that also this variant of $h$ is nondecreasing and thus, using the same reasoning as in the case $\gamma \in (1,2)$ we obtain the above result.
\end{proof}


\begin{proof}[Proof of Theorem \ref{th:RIPP_complexity}]
	In this proof we use notation $K_t$ for the number of iterations in the $t-$th epoch, large enough to turn the stopping criterion to be satisfied. We denote $x^{k,t}$ as the $k-$th IPPA iterate during $t-$th epoch. From Theorem \ref{th:general_IPP_convrate} (from Appendix B) we have that
	\begin{align}\label{rel:epoch_length}
		K_t = \left \lceil \frac{\dist_{X^*}(x^{t})}{ \delta_t} \right\rceil
	\end{align}
	iterations are sufficient to guarantee $\norm{\nabla_{\delta_{t}}F_{\mu_t}(\hat{x}^{K_t,t})} \le 5\delta_{t}^{\nabla}$  and thus the end of $t-$th epoch. Furthermore, by the triangle inequality
	\begin{align}\label{rel:lower_bound_noisygrad}
		\norm{\nabla F_{\mu_t}(x^{t+1})} - \delta_t^{\nabla} \le \norm{\nabla_{\delta_t} F_{\mu_t}(x^{t+1})} \le  5\delta_t^{\nabla},
	\end{align}
	which implies that the end of $t-$th epoch we also have $\norm{\nabla F_{\mu_t}(x^{t+1})} \le 6\delta_t^{\nabla}$.

	\vspace{5pt}
	
	\noindent Let WSM  hold and recall assumption $\dist_{X^*}(x^0) \ge \mu_0\sigma_F$. 
	For sufficiently large $t$ we show that restartation loses any effect and after a single iteration the stopping criterion of epoch $t$ is satisfied. We separate the analysis in two stages: the first stage covers the epochs that produce $x^{t+1}$ satisfying $\norm{ \nabla F_{\mu_t}(x^{t+1})} > \sigma_F$. The second one covers the rest of epochs when the gradient norms decrease below the threshold $\sigma_F$, i.e. $\norm{ \nabla F_{\mu_t}(x^{t+1})} \le \sigma_F$.
	
	In the first stage, the stopping rule $\norm{ \nabla F_{\mu_t}(x^{t+1})} \le \delta_t^{\nabla}$ limits the first stage to maximum $T_1 = \left \lceil\frac{1}{\rho}\log \left( \frac{12\delta_0}{\sigma_F}\right) \right\rceil$ epochs. 
	The total number of iterations in this stage is bounded by: $\sum_{t = 0}^{T_1}K_t  \le T_1K_0 $.
	
	For the second stage when $\norm{\nabla F_{\mu_{t-1}}(x^t)} < \sigma_F$, Lemma \ref{lemma:morenvgrad_bound_relating_with_distopt} states that $\prox_{\mu,F}(x^t) = \pi_{X^*}(x^t)$ and thus we have: 
	$\;\;	\dist_{X^*}(x^t) = \mu_{t-1}\norm{\nabla F_{\mu_{t-1}}(x^t)} \le \mu_{t-1} \delta_{t-1}^{\nabla} = \delta_{t-1}  < 2\mu_{t-1}\sigma_F = \mu_t\sigma_F.$ 
	Therefore, by Theorem \ref{th:IPPconvergence}: 
	\begin{align*}
		\dist_{X^*}(x^{t+1}) 
		&\le \max \left\{\dist_{X^*}(x^t) - K_{t}(\mu_{t}\sigma_F-\delta_t),\delta_t \right\} \\
		&\le \max \left\{\mu_{t}\sigma_F - K_t (\mu_t\sigma_F- \delta_t),\delta_t \right\}	
	\end{align*}
	which means that after a single iteration, i.e. $K_t = 1$, it is guaranteed that $\dist_{X^*}(x^{t+1}) \le \delta_t$. In this phase, the output of IPPA is in fact the only point produced in $t$-th epoch and the necessary number of epochs (or equivalently the number of IPPA iterations) is $T_2 = \BigO{\frac{1}{\rho-1}\log\left( \frac{\delta_{T_1}}{\epsilon}\right)}$.

	\vspace{10pt}

	\noindent Let $\gamma = 2$. By Lemma  \ref{lemma:deterministic_moreau_growth}$(ii)$ and convexity of $F_{\mu}$ yields $ \frac{\sigma_F}{1+2\sigma_F\mu}\dist_{X^*}^2(x) \le \langle \nabla F_{\mu}(x), x - \pi_{X^*}(x)\rangle \le \norm{\nabla F_{\mu}(x)}\dist_{X^*}(x)$. By using the inequality formed by the first and last terms, together with the relation \eqref{rel:lower_bound_noisygrad}, then at the end of $t-1$ epoch:
	\begin{align*}
		\dist_{X^*}(x^{t}) \le  \norm{\nabla F_{\mu_{t-1}}(x^t)} \left(    \frac{1}{ \mu_{t-1} \sigma_F} + 2   \right) \le 6\delta_{t-1}\left(    \frac{1}{ \mu_{t-1} \sigma_F} + 2   \right),
	\end{align*}
	suggesting that the necessary number of epochs is 
	$T = \BigO{\frac{1}{\rho}\log \left( \frac{\delta_0}{\epsilon}\right)}.$
	This fact allow us to refine $K_t$ in \eqref{rel:epoch_length} as 
	$K_t = \left\lceil   3 \cdot 2^{\rho}\left( 2  + \frac{1}{\mu_{t-1}\sigma_F}\right)     \right\rceil \quad \forall t \ge 1. \nonumber $
	Since $K_t$ is bounded, then the total number of IPPA iterations has the order	$\sum_{t = 0}^{T-1}K_t   = K_0 + \BigO{T}.$
	
	\vspace{5pt}
	
	\noindent Let $\gamma > 1$. Similarly as in the previous two cases, $\norm{\nabla F_{\mu_{t-1}}(x^t)} \le 6\delta_t^{\nabla}$ guaranteed by $t-1$ epoch further implies
	$\;\;	\dist_{X^*}(x^{t}) \overset{\eqref{gamma_all_Env_Holder_error_bound}}{\le} \Max{ \frac{12\delta_{t-1}}{\varphi(\gamma)} , \left[ \frac{ 6\delta_{t-1}^{\nabla} }{\varphi(\gamma)\sigma_F} \right]^{\frac{1}{\gamma-1}}  },$
	which suggests that the maximal number of epochs is 
	\begin{align*}
		T = \BigO{\Max{ \frac{\gamma-1}{\rho}\log \left( \frac{\delta_0}{\epsilon}\right), \frac{1}{\rho-1}\log \left( \frac{\mu_0\delta_0}{\epsilon}\right)}}.
	\end{align*}
	Now, $K_t$ of \eqref{rel:epoch_length} becomes:
	$\;	K_t = \left\lceil   \Max{ \frac{3\cdot 2^{\rho+1}}{\varphi(\gamma)} ,  D 2^{t\left[ (\rho - 1)  - \frac{\rho}{\gamma-1}\right] + \frac{\rho}{\gamma-1} }   }  \right\rceil $ for all $t \ge 1,$
	where $D = \left(\frac{6\delta_0}{\sigma_F \varphi(\gamma)}  \right)^{2/\gamma} \frac{1}{\mu_0\delta_0}$. For $\gamma \le 2$, $K_t$ is bounded, thus for $\gamma > 2$ we further estimate
	the total number of IPPA iterations by summing:
	\begin{align}
		\sum_{t = 0}^{T_1}K_t  
		& = K_0 + \sum_{t = 1}^{T_1} \left\lceil   \Max{ \frac{3\cdot 2^{\rho+1}}{\varphi(\gamma)} ,  D 2^{t\left[ (\rho - 1)  - \frac{\rho}{\gamma-1}\right] + \frac{\rho}{\gamma-1} }   }  \right\rceil \\
		& \le K_0 + T_1 +  \Max{ \frac{3\cdot 2^{\rho+1}}{\varphi(\gamma)}T_1 ,  D2^{\frac{\rho}{\gamma-1}} \sum_{t = 1}^{T} 2^{t\left[ (\rho - 1)  - \frac{\rho}{\gamma-1}\right]  }   }   \nonumber
	\end{align}
	Finally, 
	$\sum_{t = 1}^{T} 2^{t\left[ (\rho - 1)  - \frac{\rho}{\gamma-1}\right]  } 
	= \BigO{\left( \frac{\delta_0}{\epsilon}\right)^{ \Max{ \left(1  -\frac{1}{\rho} \right)(\gamma - 1) - 1 , 1 - \frac{\rho}{(\rho - 1)(\gamma - 1)} }    } }
	$.


\end{proof}

\begin{proof}[Proof of Theorem \ref{th:main_computational_theorem}]
	We keep the same notations as in the proof of Theorem \ref{th:RIPP_complexity} and redenote $\delta_t^{\nabla}:=\delta_t$. 
	By assumption $\delta_0 \ge 2L_f$ we observe that $	(4\alpha_0\mu_0L_f^2)^{\frac{1}{2}} \le 2\mu_0L_f \le \mu_0 \delta_0$.
	Since $\alpha_t = \alpha_02^{-(2\rho-1)t}$, then the inequality  $4\alpha_t\mu_tL_f^2 \le  (\delta_t)^{2} $ recursively holds for all $t \ge 0$.
	This last inequality allow Theorem \ref{th:Holder_gradients} to establish that at $t-$epoch there are enough:
	\begin{align*}
		[t = 0] &\quad	N_0 = \left\lceil 8\log\left( \frac{ \norm{\nabla F_{\mu_0}(x^0)} }{\delta_0}  \right)  \right\rceil  \\
		[t > 0]	& \quad N_t = \left\lceil 4\cdot 2^{2\rho t} \log\left( \frac{ \mu_{t-1}\delta_{t-1}}{\mu_t\delta_{t}}  \right)\right\rceil    = 
		\left\lceil 4(\rho - 1)2^{2\rho t}\right\rceil   
	\end{align*}
	PsGM iterations. Lastly, we compute the total computational cost by summing over all $N_t$.
	Recall that	at the end of $t-$th epoch RIPPA guarantees that $\norm{\nabla F_{\mu_t}(x^{t+1})} \le \delta_t^{\nabla}$. After a number of epochs of $T = \left \lceil \frac{1}{\rho} \log\left( \frac{\delta_0}{\sigma_F} \right)\right \rceil$, measuring a total number of:
	\begin{align*}
		\T_{1} = \sum\limits_{t = 0}^{T_1 - 1} N_t K_{t}
		& = N_0 K_{0} + K_{0}\sum\limits_{t = 1}^{T_1-1} \left\lceil 4(\rho - 1)2^{2\rho t}\right\rceil 
		\\
		&= \BigO{K_{0} 2^{2\rho T_1}} = \BigO{K_{0} \left(\frac{\delta_0}{\sigma_F} \right)^{2} }.
	\end{align*}
	PsGM iterations, then one has $x^{T}$ satisfying $\dist_{X^*}(x^{T}) \le \mu_{T-1}\delta_T \le \mu_{T-1}\sigma_F$. Now we evaluate the final postprocessing loop, which aims to bring the iterate $x^T$ into the $\epsilon-$suboptimality region. By Theorem \ref{th:Holder_gradients}, a single call of $PsGM \left(x^{T},x^{T},\beta_0, \mu_{T-1}, \left\lceil 2\left(\frac{L_f}{\delta_T} \right)^2 \right\rceil \right)$ guarantees that $\dist_{X^*}(\tilde{x}^{1}) \le \frac{\beta_0 L_f^2}{2\sigma_F} \le \frac{\beta_0 L_f^2}{2\delta_T} = \frac{\delta_T}{2}$. 
	In general, if $\dist_{X^*}(\tilde{x}^{k}) \le \frac{\beta_{k-1} L_f^2}{2 \sigma_F} \le \frac{\beta_{k-1} L_f^2}{2 \delta_{T}} $, Theorem \ref{th:Holder_gradients} specifies that after $\left \lceil 2\left(\frac{\dist_{X^*}(\tilde{x}^{k})}{\beta_{k} L_f} \right)^2 \right \rceil  \overset{\text{Theorem} \; \ref{th:Holder_gradients}}{\le}
	\left\lceil 2\left(\frac{L_f}{\delta_{T}} \right)^2 \right \rceil $ routine iterations, the output $\tilde{x}^{k+1}$ satisfies $\dist_{X^*}(\tilde{x}^{k+1}) \le \frac{\beta_k L_f^2}{2 \sigma_F} \le \frac{\beta_k L_f^2}{2 \delta_{T}}$. Therefore, 	by setting $N =\left\lceil 2\left(\frac{L_f}{\delta_T} \right)^2 \right\rceil, K = \lceil\log(\delta_T/\epsilon)\rceil$  and by running the final procedure \textbf{Postprocessing}($x^T,\delta_T^2/L_f^2,\mu_{T-1},N,K$), the "second phase" loop produces $\dist_{X^*}(x^{K}) \le \epsilon$. The total cost of the "second phase" loop can be computed by  
	$ \T_2 = K N = \left\lceil 2\left(\frac{L_f}{\delta_T} \right)^2 \right\rceil \left\lceil\log \left(\frac{\delta_T}{\epsilon}\right)\right\rceil.$
	Lastly, by taking into account that $\delta_t = \BigO{\sigma_F}$, then the total complexity has the order:
	$ \;	\T_1 + \T_2 = \BigO{ K_{0} \left(\frac{\delta_0}{\sigma_F} \right)^{2}  + \left(\frac{L_f}{\sigma_F} \right)^2 \log \left(\frac{\sigma_F}{\epsilon}\right)},$
	which confirms the first part of the above result.

	\noindent Now we prove the second part of the result. By assumption $\delta_0 \ge (2L_f^2)^{\frac{1}{2(1-\nu)}} \mu_0^{\frac{\nu}{1-\nu}} $ we observe that:
	\begin{align}\label{rel:induction_initial_step}
		(4\alpha_0\mu_0L_f^2)^{\frac{1}{2(1-\nu)}} \le \mu_0 \delta_0.
	\end{align}
	Further we show that, for appropriate stepsize choices $\alpha_t$, the inequality  $4\alpha_t\mu_tL_f^2 \le  (\mu_t\delta_t)^{2(1-\nu)} $ recursively holds for all $t \ge 0$ under initial condition \eqref{rel:induction_initial_step}. Indeed let $2(\rho - 1)(1-\nu) \le q-1$, then 
	\begin{align}
		4\alpha_t \mu_t L_f^2  
		& = 
		\frac{2\mu_{0}^2 L_f^2}{2^{(q-1)t}}  
		\overset{\eqref{rel:induction_initial_step}}{\le}  \frac{(\mu_0\delta_0)^{2(1-\nu)}  }{2^{2(\rho-1)(1-\nu)t}}    = (\mu_t \delta_t)^{2(1-\nu)}. \label{rel:constant_inner_condition}
	\end{align}
	The inequality \eqref{rel:constant_inner_condition} allow Theorem \ref{th:Holder_gradients} to establish the necessary inner complexity for each IPPA iteration. By using bounds from Theorem \ref{th:Holder_gradients}, at $t-$epoch there are enough:
	\begin{align*}
		[t = 0] &\quad	N_0 = \left\lceil 8\log\left( \frac{ \norm{\nabla F_{\mu_0}(x^0)} }{\delta_0}  \right)  \right\rceil  \\
		[t > 0]	& \quad N_t = \left\lceil 4\cdot 2^{(q+1)t} \log\left( \frac{\mu_{t-1} \delta_{t-1}}{\mu_t\delta_t}  \right)\right\rceil    = 
		\left\lceil 4(\rho - 1)2^{(q+1)t}\right\rceil   
	\end{align*}
	PsGM iterations. We still keep the same notations from Theorem \ref{th:RIPP_complexity}. 
	

	\noindent Let $\gamma > 1$ and recall $T = \BigO{  \Max{   \frac{\gamma-1}{\rho}\log(\mu_0\delta_0/\epsilon), \frac{1}{\rho-1}\log(\mu_0\delta_0/\epsilon)  } }.$ By following a similar reasoning, we require:
	\begin{align*}
		\sum\limits_{t = 0}^{T - 1} N_t K_{t} 
		& = N_0 K_{0} + \BigO{  
			\sum_{t = 1}^{T} 2^{t\left[ (\rho - 1)  - \frac{\rho}{\gamma-1} + q + 1\right]  } 
		}     \\
		& = N_0 K_{0}+ \BigO{ 2^{T\left[ (\rho - 1)  - \frac{\rho}{\gamma-1} + q + 1\right] }}.
	\end{align*}
	Let $\zeta = \frac{\rho - 1}{\rho}(\gamma - 1) \ge 1$, then the exponent of the last term becomes:
	\begin{align*}
		T\left[ (\rho - 1)  - \frac{\rho}{\gamma-1} + q + 1\right] 
		& = \Max{ \frac{\gamma - 1}{\rho}, \frac{1}{\rho - 1}   }  \left[ (\rho - 1)  - \frac{\rho}{\gamma-1} + q + 1\right] \\
		& = \gamma - 2 + \frac{q}{\rho}(\gamma - 1).
	\end{align*}
	Otherwise, if $\zeta < 1$ then the respective exponent turns into:
	$	T\left[ (\rho - 1)  - \frac{\rho}{\gamma-1} + q + 1\right] 
	= \frac{\rho}{\rho-1}\frac{\gamma - 2}{\gamma - 1} + \frac{q}{\rho - 1}.$

\end{proof}


\subsection{Appendix B}

\begin{lemma}\label{lemma:first_sequence}
	Let the sequence $\{u_k\}_{k \ge 0}$ satisfy: $u_{k+1} \le \alpha_k u_k + \beta_k,$ where $\alpha_k \in [0,1), \{\beta_k\}_{k \ge 0} $ nonincreasing and $\sum\limits_{i=0}^{\infty} \beta_i \le \Gamma$. Then the following bound holds:
	\begin{align*}
		u_k \le u_0 \prod\limits_{j=0}^k \alpha_j  + \Gamma \prod\limits_{j=\lceil k/2 \rceil + 1}^k \alpha_j   + \max\limits_{\lceil k/2 \rceil +1 \le i \le k}\frac{\beta_i}{1-\alpha_i}.
	\end{align*}
	Moreover, if $\alpha_k = \alpha \in [0,1)$ then:
	$\;\; u_k \le \alpha^{(k-4)/2}(u_0 + \Gamma)+ \frac{\beta_{\lceil k/2 \rceil +1}}{1-\alpha}.$ 
\end{lemma}

\begin{proof}[Proof of Lemma \ref{lemma:first_sequence}] 
	By using a simple induction we get:
	\begin{align*}
		u_{k+1} 	
		&\le \alpha_k u_k + \beta_k  \le u_0 \prod\limits_{j=0}^k \alpha_j  + \sum\limits_{i=0}^k \beta_i \prod\limits_{j=i+1}^k \alpha_j \\ 
		& = u_0 \prod\limits_{j=0}^k \alpha_j  + \sum\limits_{i=0}^{\lceil k/2 \rceil} \beta_i \prod\limits_{j=i+1}^k \alpha_j + \sum\limits_{i=\lceil k/2 \rceil+1}^k \beta_i \prod\limits_{j=i+1}^k \alpha_j
	\end{align*}
	\begin{align*}
		& \le u_0 \prod\limits_{j=0}^k \alpha_j  + \Gamma \prod\limits_{j=\lceil k/2 \rceil + 1}^k \alpha_j   + \sum\limits_{i=\lceil k/2 \rceil+1}^k \frac{\beta_i}{1-\alpha_i}(1-\alpha_i) \prod\limits_{j=i+1}^k \alpha_j \\
		& \le u_0 \prod\limits_{j=0}^k \alpha_j  + \Gamma \prod\limits_{j=\lceil k/2 \rceil + 1}^k \alpha_j   + \max\limits_{\lceil k/2 \rceil +1 \le i \le k}\frac{\beta_i}{1-\alpha_i} \sum\limits_{i=\lceil k/2 \rceil+1}^k (1-\alpha_i) \prod\limits_{j=i+1}^k \alpha_j\\ 
		& \le u_0 \prod\limits_{j=0}^k \alpha_j  + \Gamma \prod\limits_{j = \lceil k/2 \rceil + 1}^k \alpha_j   + \max\limits_{\lceil k/2 \rceil +1 \le i \le k}\frac{\beta_i}{1-\alpha_i}.
	\end{align*}
\end{proof}

%
%

\begin{lemma}\label{max_lemma}
	Let $a_1 \ge \cdots \ge a_n$ be $n$ real numbers, then the following relation holds:
	\begin{align*}
		\max\left\{ 0, a_n, a_n + a_{n-1}, \cdots, \sum\limits_{j = 1}^n a_j \right\}  = \max \left\{0,\sum\limits_{j = 1}^n a_j \right\}. 
	\end{align*} 
\end{lemma}
\begin{proof}
	Since we have: $$\max\left\{ 0, a_n, \cdots, \sum\limits_{j = k}^n a_j \right\}  = \max\left\{ \max\{0, a_n\}, \cdots, \max \left\{0,\sum\limits_{j = 1}^n a_j \right\} \right\}, $$ then it is sufficient to show that for any positive $k$:
	\begin{align}\label{max_lemma:key_ineq}
		\max\left\{ 0, \sum\limits_{j = k}^n a_j \right\} \le \max\left\{0, \sum\limits_{j = k-1}^n a_j\right\}. 
	\end{align}
	Indeed, if $a_{k-1} \ge 0$ then \eqref{max_lemma:key_ineq} results straightforward. Consider that $a_{k-1} < 0$, then by monotonicity we have: $a_j < 0$ for all $j > k-1$ and thus $\sum\limits_{j = k}^n a_j < 0$. In this case it is obvious that $\max\left\{ 0, \sum\limits_{j = k}^n a_j \right\} = \max\left\{0, \sum\limits_{j = k-1}^n a_j\right\}
	= 0$, which confirms the final above results.  
\end{proof}

\begin{theorem}\label{th:general_IPP_convrate}
	Let	$\{x^k\}_{k \ge 0}$ be the sequence generated by IPPA with inexactness criterion \eqref{inexactness_criterion}, then the following relation hold:
	\begin{align*}
		\dist_{X^*}(x^{k+1}) & \le \dist_{X^*}(x^k) - \mu \frac{ F_{\mu}(x^k) - F^* }{\dist_{X^*}(x^k)}   + \delta_k.
	\end{align*}
	Moreover, assume constant accuracy $\delta_k = \delta $. Then after at most:
	\begin{align*}
		\left\lceil \frac{\dist_{X^*}(x^0)}{\delta} \right\rceil
	\end{align*}
	iterations, a point $\tilde{x} \in \{x^0, \cdots, x^k\}$ satisfies $\norm{\nabla_{\delta} F_{\mu}(\tilde{x})} \le \frac{4\delta}{\mu}$ and $ \dist_{X^*}(\tilde{x}) \le \dist_{X^*}(x^{0})$.
\end{theorem}
\begin{proof}
	By convexity of $F$, for any $z$ we derive:
	\begin{align}
		& \norm{\prox_{\mu}^F(x^k) - z}^2  = \norm{x^k - z}^2 + 2\langle \prox_{\mu}^F(x^k) - x^k, x^k - z\rangle + \norm{\prox_{\mu}^F(x^k) - x^k}^2 \nonumber\\
		& = \norm{x^k - z}^2 - 2\mu\langle \nabla F(\prox_{\mu}^F(x^k)), \prox_{\mu}^F(x^k) -z\rangle - \norm{\prox_{\mu}^F(x^k) - x^k}^2 \nonumber\\
		& \le  \norm{x^k - z}^2 - 2\mu \left( F(\prox_{\mu}^F(x^k)) - F(z) + \frac{1}{2\mu}\norm{\prox_{\mu}^F(x^k) - x^k }^2 \right) \nonumber\\
		& = \norm{x^k - z}^2 - 2\mu \left( F_{\mu}(x^k) - F(z) \right).\label{eq:0acc_recurrence}
	\end{align}
	In order to obtain, by the triangle inequality we simply derive:
	\begin{align}\label{eq:triangle}
		\norm{x^{k+1} - z} \le \norm{\prox_{\mu}^F(x^k) - z} 
		& + \norm{\prox_{\mu}^F(x^k) - x^{k+1}} \nonumber\\  
		& \le \norm{\prox_{\mu}^F(x^k) - z} + \delta    
	\end{align}
	Finally, by taking $ z = \pi_{X^*}(x) $, then:
	\begin{align}
		\dist_{X^*}(x^{k+1}) \le \norm{x^{k+1} - \pi_{X^*}(x^k)} \overset{\eqref{eq:triangle}}{\le} \norm{\prox_{\mu}^F(x^k) - \pi_{X^*}(x^k)} + \delta   \nonumber\\
		\overset{\eqref{eq:0acc_recurrence}}{\le} \sqrt{\dist_{X^*}^2(x^k) - 2\mu \left( F_{\mu}(x^k) - F^* \right)} + \delta \label{rel:decrease_distopt_IPPA} \\
		\le \dist_{X^*}(x^k)\sqrt{1 - 2\mu \frac{ F_{\mu}(x^k) - F^* }{\dist_{X^*}^2(x^k)}   } + \delta \nonumber\\
		\le \dist_{X^*}(x^k)\left(1 - \mu \frac{ F_{\mu}(x^k) - F^* }{\dist_{X^*}^2(x^k)}   \right) + \delta, \nonumber
	\end{align}
	where in the last inequality we used the fact $\sqrt{1-2a}\le 1-a$.
	The last inequality leads to the first part from our result:
	\begin{align}\label{rel:prelim_recurrence_(B)}
		\dist_{X^*}(x^{k+1}) & \le \dist_{X^*}(x^k) - \mu \frac{ F_{\mu}(x^k) - F^* }{\dist_{X^*}(x^k)}   + \delta_k.
	\end{align}
	Assume that 
	\begin{align}\label{rel:Fmu_assumption}
	\frac{ F_{\mu}(x^0) - F^* }{\dist_{X^*}(x^0)} \ge \frac{\delta}{\mu}
	\end{align}
	and denote $K = \min\{k \ge 0 \;:\; \dist_{X^*}(x^{k+1})  \ge \dist_{X^*}(x^k)  \}$. Then \eqref{rel:prelim_recurrence_(B)} has two consequences. First, obviously for all $k < K$:
	\begin{align*}
		F_{\mu}(x^k) - F^* \le \frac{1}{\mu}\left(\dist_{X^*}^2(x^{k}) - \dist_{X^*}^2(x^{k+1}) \right) + \frac{\delta}{\mu}\dist_{X^*}(x^0).
	\end{align*}
	By further summing over the history  we obtain:
	\begin{align}\label{rel:1/k_MG_rate}
		F_{\mu}(\hat{x}^k) - F^* \le \min\limits_{0 \le i \le k} F_{\mu}(x^i) - F^* & \le \frac{1}{k+1}\sum\limits_{i=0}^k F_{\mu}(x^i) - F^* \nonumber\\ 	
		&\le \frac{\dist_{X^*}^2(x^{0}) }{\mu (k+1)} + \frac{\delta}{\mu}\dist_{X^*}(x^0).
	\end{align}
	
	\noindent Second, since $K$ is the first iteration at which the residual optimal distance increases, then $\dist_{X^*}(x^K) \le \dist_{X^*}(x^{K-1}) \le \cdots \le \dist_{X^*}(x^0)$ and \eqref{rel:prelim_recurrence_(B)} guarantees:
	\begin{align*}
		F_{\mu}(x^K) - F^* \le \frac{\delta}{\mu}\dist_{X^*}(x^K) \le \frac{\delta}{\mu}\dist_{X^*}(x^0).
	\end{align*}
	By unifying both cases we conclude that after at most: $K_{\delta} = \frac{\dist_{X^*}(x^0)}{\delta}$ iterations the threshold: $F_{\mu}(x^{K_{\delta}}) - F^* \le  \frac{2\delta}{\mu}\dist_{X^*}(x^0)$ is reached. Notice that if \eqref{rel:Fmu_assumption} do not hold, then $K_{\delta} = 0$.
	
	\vspace{5pt}
	
	\noindent Now we use the same arguments from \cite[Sec. I]{Nes:12Optima} to bound the norm of the gradients. Observe that the Lipschitz gradients property of $F_{\mu}$ leads to:
	\begin{align}
		& F_{\mu}(\hat{x}^{k+1})  \le F_{\mu}(\hat{x}^k - \mu \nabla_{\delta} F(\hat{x}^k)) \nonumber \\
		& = F_{\mu}(\hat{x}^k) - \mu \langle \nabla F_{\mu}(\hat{x}^k),  \nabla_{\delta} F_{\mu}(\hat{x}^k) \rangle + \frac{\mu}{2}\norm{\nabla_{\delta} F_{\mu}(\hat{x}^k)}^2 \nonumber\\
		& = F_{\mu}(\hat{x}^k) + \mu \langle \nabla_{\delta} F_{\mu}(\hat{x}^k) - \nabla F_{\mu}(\hat{x}^k),  \nabla_{\delta} F_{\mu}(\hat{x}^k) \rangle - \frac{\mu}{2}\norm{\nabla_{\delta} F_{\mu}(\hat{x}^k)}^2 \nonumber\\
		& = F_{\mu}(\hat{x}^k) - \frac{\mu}{4}\norm{\nabla_{\delta} F_{\mu}(\hat{x}^k)}^2 + \mu \langle \nabla_{\delta} F_{\mu}(\hat{x}^k) - \nabla F_{\mu}(\hat{x}^k),  \nabla_{\delta} F_{\mu}(\hat{x}^k) \rangle - \frac{\mu}{4}\norm{\nabla_{\delta} F_{\mu}(\hat{x}^k)}^2 \nonumber\\
		& \le F_{\mu}(\hat{x}^k) - \frac{\mu}{4}\norm{\nabla_{\delta} F_{\mu}(\hat{x}^k)}^2 + \mu \norm{\nabla_{\delta} F_{\mu}(\hat{x}^k) - \nabla F_{\mu}(\hat{x}^k)}^2 \nonumber\\
		& = \! F_{\mu}(\hat{x}^k) \!-\! \frac{\mu}{4}\norm{\nabla_{\delta} F_{\mu}(\hat{x}^k)}^2 \!+\! \frac{\delta^2}{\mu} \! = \! F_{\mu}(\hat{x}^{k/2}) - \frac{k\mu}{8}\norm{\nabla_{\delta} F_{\mu}(\hat{x}^{k/2})}^2 + \frac{k\delta^2}{2\mu}. \label{rel:inexact_descent_lemma}
	\end{align}
By using  \eqref{rel:1/k_MG_rate} into \eqref{rel:inexact_descent_lemma}, then for $k \ge K_{\delta}$
\begin{align*}
\norm{\nabla_{\delta} F_{\mu}(\hat{x}^k)}^2  
& \le \frac{4 (F_{\mu}(\hat{x}^k) - F^*)}{k\mu} + \frac{\delta^2}{\mu}  \le \frac{8 \dist_{X^*}(x^0)\delta}{k\mu^2} + \frac{4\delta^2}{\mu^2} \\
& \le \frac{8 \delta^2}{\mu^2} + \frac{4\delta^2}{\mu^2} =  \frac{12\delta^2}{\mu^2}.
\end{align*}



\end{proof}

\begin{lemma}\label{lemma:decrease}
	Let $\gamma-$HG holds for the objective function $F$. Then IPPA sequence $\{x^k\}_{k \ge 0}$ with variable accuracies $\delta_k$, satisfies the following reccurences:	
	
	\noindent $(i)$ Under weak sharp minima $\gamma = 1$
	\begin{align*}
		\dist_{X^*}(x^{k+1}) \le \max\left\{ \dist_{X^*}(x^k) - \mu\sigma_F,  0 \right\} + \delta_k
	\end{align*}
	
	\noindent $(ii)$ Under quadratic growth $\gamma = 2$
	\begin{align*}
		\dist_{X^*}(x^{k+1}) \le \frac{1}{\sqrt{1 + 2\mu \sigma_F}} \dist_{X^*}(x^k)+ \delta_k
	\end{align*}
	
	\noindent $(iii)$ Under general Holderian growth $\gamma \ge 1$
	\begin{align*}
		\dist_{X^*}(x^{k+1}) \le \max\left\{ \dist_{X^*}(x^k) - \mu \varphi(\gamma) \sigma_F \dist_{X^*}^{\gamma-1}(x^k), \left(1 - \frac{\varphi(\gamma)}{2}\right)\dist_{X^*}(x^k) \right\}  + \delta_k,
	\end{align*}
\end{lemma}
\begin{proof}
	$(i)$ Assume $\dist_{X^*}(x^k) > \sigma_F \mu$ then from (the proof of) Theorem \ref{th:general_IPP_convrate} and Lemma \ref{lemma:deterministic_moreau_growth}:
	\begin{align*}
		\dist_{X^*}(x^{k+1}) 
		&\le \sqrt{\dist_{X^*}^2(x^k) - 2\mu \left( F_{\mu}(x^k) - F^*\right)} + \delta_k \\
		&\le \sqrt{\dist_{X^*}^2(x^k) - 2\mu \left( \sigma_F\dist_{X^*}(x^k)- \frac{\sigma^2_F\mu}{2}\right) } + \delta_k \\
		& =  \sqrt{\left( \dist_{X^*}(x^k) - \mu \sigma_F\right)^2 } + \delta_k =  \dist_{X^*}(x^k) - \left(\mu \sigma_F - \delta_k \right). 
	\end{align*}
	In short,
	\begin{align*}
		\dist_{X^*}(x^{k+1}) \le \begin{cases} \dist_{X^*}(x^k) - (\mu\sigma_F - \delta_k), & \text{if} \; \dist_{X^*}(x^k)> \sigma_F\mu \\
			\delta_k, & \text{if} \; \dist_{X^*}(x^k) \le  \sigma_F\mu  \end{cases}
	\end{align*}
	
	\noindent $(ii)$ By using the same relations in the case $\gamma = 2$, then:
	\begin{align*}
		\dist_{X^*}(x^{k+1}) 
		& \le \sqrt{\dist_{X^*}^2(x^k) - 2\mu \left( F_{\mu}(x^k) - F^*\right)} + \delta_k \\
		&\le \sqrt{\dist_{X^*}^2(x^k) - \frac{2\mu\sigma_F}{1 + 2\mu\sigma_F} \dist_{X^*}^2(x^k) } + \delta_k  = \frac{1}{\sqrt{1 + 2\mu\sigma_F}} \dist_{X^*}(x^k)+ \delta_k.
	\end{align*}

	\noindent $(iii)$ Under Holderian growth, similarly Theorem \ref{th:general_IPP_convrate} and Lemma \ref{lemma:deterministic_moreau_growth} lead to:
	\begin{align*}
		& \dist_{X^*}(x^{k+1}) 
		 \le \dist_{X^*}(x^k) - \mu \frac{ F_{\mu}(x^k) - F^*}{\dist_{X^*}(x^k)} + \delta_k \\
		&\le \dist_{X^*}(x^k) - \mu \varphi(\gamma) \min\left\{ \sigma_F \dist_{X^*}^{\gamma-1}(x^k), \frac{1}{2\mu}\dist_{X^*}(x^k) \right\}  + \delta_k \\
		& = \max\left\{ \dist_{X^*}(x^k) - \mu \varphi(\gamma) \sigma_F \dist_{X^*}^{\gamma-1}(x^k), (1 - \varphi(\gamma)/2)\dist_{X^*}(x^k) \right\}  + \delta_k.
	\end{align*}
\end{proof}


\begin{theorem}	\label{th:central_recurrence}
	Let $\alpha,\rho>0, \beta \in (0,1)$ and $h(r) = \max\{ r - \alpha r^{\rho}, \beta r \}$. Then the sequence $r_{k+1} = h(r_k)$ 
	satisfies:
	
	\noindent $(i)$ For $\rho \in (0,1)$:
	\begin{align}\label{r_rate_rho<1}
		r_{k}   &\le 
		\begin{cases} 
\left(1 - \frac{\alpha}{2r_0^{1-\rho}} \right)^k \left[r_0 - k \frac{\alpha}{2} \left( \frac{\alpha}{1-\beta}\right)^{\frac{\rho}{1-\rho}} \right], & \text{if} \;\; r_k > \left( \frac{\alpha}{1-\beta}\right)^{\frac{1}{1-\rho}}  
\\
			\beta^{k-k_0-1}\left( \frac{\alpha}{1-\beta}\right)^{\frac{1}{1-\rho}} , & \text{if} \;\; r_k \le \left( \frac{\alpha}{1-\beta}\right)^{\frac{1}{1-\rho}}.
		\end{cases}
	\end{align}
	
	\noindent $(ii)$ For $\rho \ge 1$:
	\begin{align}\label{r_rate_rho>1}
		r_k \le  
		\begin{cases}
			\hat{\beta}^k r_0, &\text{if} \; r_k > \left( \frac{1-\hat{\beta}}{\alpha}\right)^{\frac{1}{\rho - 1}}\\	
		\left[\frac{1}{\frac{1}{\min\{r_0^{\rho - 1}, \frac{1 - \hat{\beta}}{\alpha}  \}} + (\rho-1) (k-k_0) \alpha}\right]^{\frac{1}{\rho-1}} , &\text{if} \; r_k \le \left( \frac{1-\hat{\beta}}{\alpha}\right)^{\frac{1}{\rho - 1}},	
		\end{cases}
	\end{align}
	where $k_0 = \left\{\min\limits_{k \ge 0} \; k: \; r_k \le \left( \frac{1-\hat{\beta}}{\alpha}\right)^{\frac{1}{\rho - 1}} \right\}, \hat{\beta} = \max\left\{\beta, 1- 1/\rho \right\}$.
\end{theorem}	
\begin{proof}
	Denote $g(r) = r - \alpha r^{\rho}$. 
	
\noindent \textbf{Consider $\rho \in (0,1)$}. In this case, note that $g$ is nondecreasing and thus also $h$ is nondecreasing, we have:
	\begin{align*}
		r_{k+1} &= 
		\begin{cases} 
			r_k - \alpha r_k^{\rho}, & \text{if} \;\; r_k > \left( \frac{\alpha}{1-\beta}\right)^{\frac{1}{1-\rho}}  \\
			\beta r_k, & \text{if} \;\; r_k \le \left( \frac{\alpha}{1-\beta}\right)^{\frac{1}{1-\rho}}
		\end{cases}  
	\end{align*}
	Observe that if $r_k > \left( \frac{\alpha}{1-\beta}\right)^{\frac{1}{1-\rho}}$ then, by using the monotonicity of $r_k$, we can further derive another bound:
	\begin{align*}
		r_{k+1} & 
		\le r_k - \frac{\alpha}{2} r_k^{\rho} - \frac{\alpha}{2} r_k^{\rho} 
		\le \left(1 - \frac{\alpha}{2r_k^{1-\rho}}  \right) r_k - \frac{\alpha}{2}\left( \frac{\alpha}{1-\beta}\right)^{\frac{1}{1-\rho}} \\
		&\le \left(1 - \frac{\alpha}{2r_0^{1-\rho}}  \right) r_k - \frac{\alpha}{2}\left( \frac{\alpha}{1-\beta}\right)^{\frac{1}{1-\rho}}.
	\end{align*}	
Any given sequence $u_k$ satisfying the recurrence $u_{k+1} \le (1-\xi)u_k - c$ can be further bounded as: $u_{k+1} \le (1-\xi)^k u_0 - c\sum_{i=0}^{k-1} (1-\xi)^i \le (1-\xi)^k u_0 - c\sum_{i=0}^{k-1} (1-\xi)^k = (1-\xi)^k [u_0 - k c].$ Thus, by apply similar arguments to our sequence $r_k$ we refined the above bound as follows:
	\begin{align*}
		r_{k+1}   &\le 
		\begin{cases} 
			\left(1 - \frac{\alpha}{2r_0^{1-\rho}} \right)^k \left[r_0 - k \frac{\alpha}{2} \left( \frac{\alpha}{1-\beta}\right)^{\frac{\rho}{1-\rho}} \right], & \text{if} \;\; r_k > \left( \frac{\alpha}{1-\beta}\right)^{\frac{1}{1-\rho}}  \\
			\beta^{k-k_0}\min\left\{r_0,\left( \frac{\alpha}{1-\beta}\right)^{\frac{1}{1-\rho}}\right\} , & \text{if} \;\; r_k \le \left( \frac{\alpha}{1-\beta}\right)^{\frac{1}{1-\rho}}.
		\end{cases}
	\end{align*}		

	\vspace{5pt}

	\noindent \textbf{Now consider $\rho > 1$}. In this case, on one hand, the function $g$ is nondecreasing only on $\left(0,\left( \frac{1}{\alpha \rho} \right)^{\frac{1}{\rho -1}} \right]$. On the other hand, for $r \ge \left( \frac{1}{\alpha \rho} \right)^{\frac{1}{\rho -1}}$ it is easy to see that $g(r) \le \left(1- \frac{1}{\rho} \right)r$. These two observations lead to:
	\begin{align*}
		h(r) & = \max\{r - \alpha r^{\rho},\beta r\}  \le \max \left\{r - \alpha r^{\rho},\left(1- \frac{1}{\rho} \right)r, \beta r \right\} \\ 
		& = \max \left\{r - \alpha r^{\rho}, \hat{\beta} r \right\}:= \hat{h}(r),  
	\end{align*}
	where  $\hat{\beta} = \max\left\{1- \frac{1}{\rho}, \beta \right\} $. Since $\hat{h}$ is nondecreasing, then $r_k \le \hat{h}^{(k)}(r_0)$. In order to determine the clear convergence rate of $r_k$, based on \cite[Lemma 6, Section 2.2]{Pol:78} we make a last observation:
	\begin{align}\label{rel:Poliak}
		g^{(k)}(r) \le \frac{r}{\left(1 + (\rho-1) r^{\rho-1} k \alpha \right)^{\frac{1}{\rho-1}}} \le \left[\frac{1}{\frac{1}{r^{\rho - 1}} + (\rho-1) k \alpha }\right]^{\frac{1}{\rho-1}}
	\end{align}
	Using this final bound, we are able to deduce the explicit convergence rate:
	\begin{align*}
		r_k \le \hat{h}^{(k)}(r_0) 
		&\le 
		\begin{cases}
			\hat{\beta}^k r_0, &\text{if} \; r_k > \left( \frac{1-\hat{\beta}}{\alpha}\right)^{\frac{1}{\rho - 1}}\\	
			g^{(k-k_0)}(r_{k_0}) , &\text{if} \; r_k \le \left( \frac{1-\hat{\beta}}{\alpha}\right)^{\frac{1}{\rho - 1}}	
		\end{cases}	\nonumber\\
		&\overset{\eqref{rel:Poliak}}{\le} 
		\begin{cases}
		\hat{\beta}^k r_0, &\text{if} \; r_k > \left( 	\frac{1-\hat{\beta}}{\alpha}\right)^{\frac{1}{\rho - 1}}\\	
		\left[\frac{1}{\frac{1}{\min\{r_0^{\rho - 1}, \frac{1 - \hat{\beta}}{\alpha}  \}} + (\rho-1) (k-k_0) \alpha}\right]^{\frac{1}{\rho-1}} , &\text{if} \; r_k \le \left( \frac{1-\hat{\beta}}{\alpha}\right)^{\frac{1}{\rho - 1}}.	
	\end{cases}	
	\end{align*}
\end{proof}


\begin{corrolary}\label{corr:exact_complexity}
Under the assumptions of Theorem \ref{th:central_recurrence}, let $r_{k+1} = h(r_k)$ and $\epsilon > 0$. The sequence $r_k$ attains the threshold $r_k \le \epsilon $ after the following number of iterations:

\noindent $(i)$ For $\rho \in (0,1)$:
\begin{align}\label{exact_complexity_r_[1,2]}
K \ge \Min{    \frac{2r_0^{1-\rho}}{\alpha} \log \left(\frac{r_0}{\max\{\epsilon, \tau^{\rho}\alpha/2\}} \right), \frac{2r_0}{\tau^{\rho}\alpha}    } + \frac{1}{\beta} \log\left( \frac{\Min{r_0,\tau}}{\epsilon} \right)
\end{align}
	
	\noindent $(ii)$ For $\rho \ge 1$:
	\begin{align}\label{exact_complexity_r_[2,infty]}
K \ge\frac{1}{\hat{\beta}} \log\left( \frac{r_0}{\tau(\hat{\beta})} \right) + \frac{1}{(\rho-1)\alpha} \left( \frac{1}{\epsilon^{\rho-1}} -  \frac{1}{\Min{r_0,\tau(\hat{\beta})}^{\rho-1}} \right),
	\end{align}
where $\tau (\beta) = \left(\frac{\alpha}{1-\beta} \right)^{\frac{1}{1-\rho}}$.
\end{corrolary}	

\begin{proof}
$(i)$ Let $\rho \in (0,1)$. In the first regime of \eqref{r_rate_rho<1}, when $r_k > \tau(\beta)$, there are necessary at most:
	\begin{align}
	K_1^{(0,1)} \ge \Min{    \frac{2r_0^{1-\rho}}{\alpha} \log \left(\frac{r_0}{\max\{\epsilon, \tau^{\rho}\alpha/2\}} \right), \frac{2r_0}{\tau^{\rho}\alpha}    }
\end{align}
iterations, while the second regime, i.e. $r_k \le \tau(\beta)$,  has a length of at most:
	\begin{align}
	K_2^{(0,1)} \ge \frac{1}{\beta} \log\left( \frac{\Min{r_0,\tau}}{\epsilon} \right)
\end{align}
iterations to reach $r_k \le \epsilon$. An upper margin on the total number of iterations is $K_1^{(0,1)} + K_2^{(0,1)}$.

\vspace{5pt}

\noindent $(ii)$ Let $ \rho > 1 $. Similarly, the first regime when $r_k > \tau(\hat{\beta})$ has a maximal length of:
$K_1^{(1,\infty)} \ge \frac{1}{\hat{\beta}} \log\left( \frac{r_0}{\tau(\hat{\beta})} \right).$
The second regime, while $r_k \le \tau(\hat{\beta})$, requires at most:
$K_2^{(1,\infty)} \ge \frac{1}{(\rho-1)\alpha} \left( \frac{1}{\epsilon^{\rho-1}} -  \frac{1}{\Min{r_0,\tau(\hat{\beta})}^{\rho-1}} \right) $ iteration to get $r_k \le \epsilon$.
\end{proof}


\begin{lemma}\label{lemma:prelim_conv_rate}
Let $\alpha, \rho>0, \beta \in (0,1)$. Let the sequence $\{r_k,\delta_k\}_{k \ge 0}$ satisfy the recurrence:
\begin{align*}
r_{k+1} \le \max\{r_k -\alpha r_k^{\rho}, \beta r_k \}+ \delta_k.
\end{align*}
For $\rho \in (0,1)$, let $h(r) = \max\left\{r - \frac{\alpha}{2} r^{\rho} , \frac{1+\beta}{2} r \right\}$, then:
$$r_k \le \max\left\{h^{(k)}(r_0)),h^{(k-1)}\left(\hat{\delta}_1 \right), \cdots, h\left(\hat{\delta}_{k-1} \right),\hat{\delta}_{k} \right\} $$
For $\rho \ge 1$, let $\hat{h}(r)=\max\left\{h(r), \left( 1- \frac{1}{\rho}\right) r \right\}$, then:
$$r_k \le \max\left\{\hat{h}^{(k)}(r_0)),\hat{h}^{(k-1)}\left(\hat{\delta}_1 \right), \cdots, \hat{h}\left(\hat{\delta}_{k-1} \right),\hat{\delta}_{k} \right\},$$
where $\hat{\delta}_k = \max\left\{\left(\frac{2\delta_k}{\alpha}\right)^{\frac{1}{\rho}},\frac{2\delta_k}{1-\beta} \right\}$.
\end{lemma}
\begin{proof}
	Starting from the recurrence we get:
	\begin{align*}
		r_{k+1} 
		& \le \max\{r_k -\alpha r_k^{\rho}, \beta r_k \}+ \delta_k  =  \max\{r_k -\alpha r_k^{\rho} + \delta_k, \beta r_k+ \delta_k \} \\
		& =  \max\left\{r_k - \frac{\alpha}{2} r_k^{\rho} + \left(\delta_k - \frac{\alpha}{2} r_k^{\rho}\right), \frac{1+\beta}{2} r_k+ \left(\delta_k - \frac{1-\beta}{2}r_k\right) \right\}
	\end{align*}
	If $\delta_k \le \min\left\{\frac{\alpha}{2} r_k^{\rho},\frac{1-\beta}{2}r_k\right\} $, or equivalently $r_k \ge \max\left\{\left(\frac{2\delta_k}{\alpha}\right)^{\frac{1}{\rho}},\frac{2\delta_k}{1-\beta}\right\}$, then we recover the recurrence:
	\begin{align}\label{rel:max1}
		r_{k+1} 
		& \le  \max\left\{r_k - \frac{\alpha}{2} r_k^{\rho} , \frac{1+\beta}{2} r_k \right\}
	\end{align}
	Otherwise, clearly 
	\begin{align}\label{rel:max2}
		r_k \le \max\left\{\left(\frac{2\delta_{k}}{\alpha}\right)^{\frac{1}{\rho}},\frac{2\delta_{k}}{1-\beta}\right\}
	\end{align}
	By combining both bounds \eqref{rel:max1} and \eqref{rel:max2}, we obtain:
	\begin{align}\label{rel:combined_bounds_rho_(0,1)}
		r_{k+1} 
		& \le  \max\left\{r_k - \frac{\alpha}{2} r_k^{\rho} , \frac{1+\beta}{2} r_k, \left(\frac{2\delta_{k+1}}{\alpha}\right)^{\frac{1}{\rho}},\frac{2\delta_{k+1}}{1-\beta} \right\}.
	\end{align}
Denote $h(r) = \max\left\{r - \frac{\alpha}{2} r^{\rho} , \frac{1+\beta}{2} r \right\}$ and $\hat{\delta}_k = \max\left\{\left(\frac{2\delta_k}{\alpha}\right)^{\frac{1}{\rho}},\frac{2\delta_k}{1-\beta} \right\}$. For $\rho \in (0,1)$, since both functions $ r \mapsto  r - \alpha r^{\rho}$ and $r \mapsto \frac{1+\beta}{2} r $ are nondecreasing, then $h$ is nondecreasing.
This fact allows to apply the following induction to \eqref{rel:combined_bounds_rho_(0,1)}:
	\begin{align}
	r_{k+1} 
	&\le \max\left\{h(r_k), \hat{\delta}_{k+1}\right\} \le  \max\left\{h \left( \Max{h(r_{k-1}), \hat{\delta}_{k}} \right), \hat{\delta}_{k+1}\right\}  \nonumber\\
	&\le \max\left\{h(h(r_{k-1})),h\left(\hat{\delta}_k \right),\hat{\delta}_{k+1}\right\} \nonumber\\
	& \cdots \nonumber\\
	&\le \max\left\{h^{(k+1)}(r_0)),h^{(k)}\left(\hat{\delta}_1 \right), \cdots, h\left(\hat{\delta}_k \right),\hat{\delta}_{k+1} \right\}. \label{rel:composition_induction}
	\end{align}
	In the second case when $\rho \ge 1$, the corresponding recurrence function $\hat{h}(r) = \max\left\{r - \frac{\alpha}{2} r^{\rho} , \left( 1- \frac{1}{\rho}\right) r, \frac{1+\beta}{2} r \right\}$ is again nondecreasing. Indeed, here $ r \mapsto  r - \alpha r^{\rho}$ is nondecreasing only when $r \le \left( \frac{1}{\alpha \rho}\right)^{\frac{1}{\rho-1}}$. However, if $r > \left( \frac{1}{\alpha \rho}\right)^{\frac{1}{\rho-1}}$,  then $\hat{h}(r) = \max\left\{1- \frac{1}{\rho}, \frac{1+\beta}{2}  \right\}r$ which is also nondecreasing. Thus we get our claim. The monotonicity of $\hat{h}$ and majorization $\hat{h}(r) \ge h(r)$, allow us to	obtain by a similar induction an analog relation to \eqref{rel:composition_induction}, which holds with $\hat{h}$.
\end{proof}

\vspace{0.1cm}


\begin{proof}[Proof of Theorem \ref{th:IPPconvergence}]

\noindent $(i)$ Denote $r_{k} = \dist_{X^*}(x^k)$. Since $\delta_k \le \delta_{k-1}$, then by rolling the recurrence in Lemma \ref{lemma:decrease} we get:
	\begin{align}
		r_{k+1} 
		&\le \max\left\{r_k - (\mu\sigma_F - \delta_k), \delta_k \right\} \nonumber \\
		&\le \max\left\{r_{k-1} - [2\mu\sigma_F - \delta_k- \delta_{k-1}], \delta_k + \delta_{k-1} - \mu\sigma_F, \delta_k \right\} \nonumber \\
		&\le \max\left\{r_{0} - \sum\limits_{i=0}^{k}(\mu\sigma_F - \delta_i), \delta_k + \max\left\{0, \delta_{k-1} - \mu\sigma_F, \cdots, \sum\limits_{i=0}^{k-1} (\delta_{i} - \mu\sigma_F) \right\} \right\} \label{main_recurrences}
	\end{align}
	
	\noindent By using the Lemma \ref{max_lemma}, then \eqref{main_recurrences} can be refined as: 
	\begin{align*}
		r_{k+1} 
		&\le \max\left\{r_{0} - \sum\limits_{i=0}^{k}(\mu\sigma_F - \delta_i), \delta_k + \max\left\{0, \delta_{k-1} - \mu\sigma_F, \cdots, \sum\limits_{i=0}^{k-1} (\delta_{i} - \mu\sigma_F) \right\} \right\} \\
		& \overset{\text{Lemma} \; \ref{max_lemma}}{\le} \max\left\{r_{0} - \sum\limits_{i=0}^{k}(\mu\sigma_F - \delta_i), \delta_k + \max\left\{0, \sum\limits_{i=0}^{k-1} \delta_{i} - \mu\sigma_F \right\} \right\} \\ 
		& \le \max\left\{r_{0} - \sum\limits_{i=0}^{k}(\mu\sigma_F - \delta_i), \max\left\{\delta_k, \mu\sigma_F+ \sum\limits_{i=0}^{k} \delta_{i} - \mu\sigma_F \right\} \right\} \\ 
		&= \max\left\{\max\{r_{0},\mu\sigma_F\} - \sum\limits_{i=0}^{k}(\mu\sigma_F - \delta_i), \delta_k  \right\}. 
	\end{align*}
	
	\vspace{5pt}
	
	\noindent $(ii)$ Denote $\theta = \frac{1}{(1 + 2 \sigma_F \mu)^{1/2}}$. From Lemmas \ref{lemma:decrease} and \ref{lemma:first_sequence} we derive that:
	\begin{align*}
		 \dist_{X^*}(x^{k}) 
		\le   \theta \dist_{X^*}(x^{k-1}) +   \delta_{k-1} 
		 \overset{\text{Lemma} \; \ref{lemma:first_sequence}}{\le}    \theta^{\frac{k-4}{2}} \left(\dist_{X^*}(x^{0}) +  \Gamma\right)   + \frac{\delta_{\lceil k/2 \rceil + 1}}{1-\theta}.
	\end{align*}
	
	\vspace{5pt}

	
	\noindent $(iii)$ First consider $\gamma \in [1,2)$ and let $h(r) = \Max{r - \frac{\mu\varphi(\gamma)\sigma_F}{2} r^{\gamma-1}, \frac{1+\sqrt{1-\varphi(\gamma)}}{2}r }$. Then by Lemmas \ref{lemma:decrease} and \ref{lemma:prelim_conv_rate}, we have that:
	\begin{align}\label{rel:prelim_conv_rate_1}
	r_{k+1} 
	&\le \max\left\{h^{(k)}(r_0)),h^{(k-1)}\left(\hat{\delta}_1 \right), \cdots, h\left(\hat{\delta}_{k-1} \right),\hat{\delta}_{k} \right\},
	\end{align}
	where $\hat{\delta}_k =  \max\left\{\left(\frac{2\delta_k}{\mu\varphi(\gamma)\sigma_F}\right)^{\frac{1}{\rho}},\frac{2\delta_k}{1-\sqrt{1-\varphi(\gamma)}} \right\}$. 
	Let some $u_k = h^{(k)}(u_0)$ and $\bar{\delta}_k = \Max{\hat{\delta}_k, h(\bar{\delta}_{k-1})}$. Then, since $h$ is nondecreasing, we get:
	\begin{align*}
		r_{k+1} 
		\le \max\left\{h^{(k)}(r_0)),h^{(k-1)}\left(\hat{\delta}_1 \right), \cdots, h\left(\hat{\delta}_{k-1} \right),\hat{\delta}_{k} \right\} = \Max{u_k,\bar{\delta}_k },
	\end{align*}
	Finally, by using the convergence rate upper bounds from the Theorem \ref{th:central_recurrence}, we can further find out an the convergence rate order of $u_k$.
	We can appeal to a similar argument when $\gamma \ge 2$, by using the nondecreasing function $\hat{h}(r) = \max\left\{h(r), \left( 1- \frac{1}{\rho}\right)r \right\}$, instead of $h$.

\end{proof}

\begin{theorem}\label{th:Holder_gradients}
	\label{th:complexity_inner}
	Let the function $f$ having $\nu-$Holder continuous gradients with constant $L_f$ and $\nu \in [0,1]$. Also let $\mu > 0, \alpha \le \Min{ \frac{\mu}{2}, \frac{\delta^{2(1-\nu)}}{4\mu L_f^2} }, z^0 \in \dom(\psi)$ and 
	\begin{align}\label{PsGM_routine_const_complexity}
		N \ge \left \lceil \frac{4\mu}{\alpha}\log\left( \frac{   \norm{z^0-\prox_{\mu}^F(x)}   }{\delta}\right) \right \rceil  
	\end{align} 
	then PsGM($z^0,x, \alpha, \mu, N$) outputs $z^N$ such that $\norm{z^{N}  - \prox_{\mu}^F(x)} \le \delta$.
	
	\noindent Moreover, assume particularly that $\nu = 0$ and $F$ satisfies WSM with constant $\sigma_f$. Also let $\alpha \in (0, \mu/2 ]$, $Q$ be a closed convex feasible set and $\psi = \iota_Q$ its indicator function. 
	If
	\begin{align}\label{PsGM_routine_const_complexity}
		\dist_{X^*}(x) \le \mu\sigma_F \qquad \text{and} \qquad	N \ge \left \lceil 2\left(\frac{\dist_{X^*}(x)}{\alpha L_f} \right)^2 \right \rceil  
	\end{align} 
	then PsGM($x,x, \alpha, \mu, N$) outputs $z^N$ satisfying $\dist_{X^*}(z^N) \le \frac{\alpha L_f^2}{2\sigma_F}$.
\end{theorem}

\begin{proof}
	For brevity we avoid the counter $k$ and denote $z(x) := \prox_{\mu}^F(x), z^+ := \prox_{\mu}^\psi\left(z  - \alpha \left[f'(z)  +\frac{1}{\mu}(z - x) \right] \right)$. Recall the optimality condition:
	\begin{align}\label{rel:strong_convexity_inner}
		z(x) = \prox_{\mu}^\psi\left(z(x)  - \alpha \left[f'(z(x))  +\frac{1}{\mu}(z(x)-x) \right] \right).
	\end{align}
	By using $\nu-$Holder continuity then we get:
	\begin{align}\label{rel:bound_holdgrad}
		\norm{f'(z) -f'(z(x))} \le L_f \norm{z-z(x)}^{\nu}  \quad \forall z.
	\end{align}
	Then the following recurrence holds:
	\begin{align}
		&	\norm{z^+-z(x)}^2 \nonumber\\
		& \!\! \overset{\eqref{rel:strong_convexity_inner}}{=} \!\!\left\| \prox_{\mu}^\psi \!\!\left( z\!\! - \!\!\alpha \left[f'(z) \!+\! \frac{1}{\mu}(z\!\!-\!\!x) \right] \right) \!\!-\!\! \prox_{\mu}^\psi \left( z(x) \!\!-\!\! \alpha \left[f'(z(x))\! +\! \frac{1}{\mu}(z(x) \!\!-\!\! x) \right] \right) \right\|^2 \nonumber\\
		& \le \left\| \left(1 - \frac{\alpha}{\mu}\right)(z-z(x)) + \alpha[f'(z(x)) - f'(z)] \right\|^2 \nonumber\\
		&= \left(1 - \frac{\alpha}{\mu}\right)^2\norm{z - z(x)}^2 - 2\alpha \left(1 - \frac{\alpha}{\mu}\right) \langle f'(z) - f'(z(x)), z - z(x)\rangle \nonumber\\ 
		&	\hspace{6cm} + \alpha^2\norm{f'(z) -f'(z(x)) }^2 \label{rel:subgrad_prelim_recurrence}\\
		&\overset{\eqref{rel:bound_holdgrad}}{\le} \left(1 - \frac{\alpha}{\mu}\right)^2\norm{z - z(x)}^2 + \alpha^2 L_f^2\norm{ z-z(x)}^{2\nu}.\nonumber
	\end{align}
	Obviously, a small stepsize $\alpha < \mu$ yields $\left(1 - \frac{\alpha}{\mu}\right)^2 \le 1 - \frac{\alpha}{\mu}$.
	If the squared residual is dominant, i.e.
	\begin{align}\label{rel:linconv_zone}
		\norm{z - z(x)} \ge \delta^{} \ge \left(2\alpha\mu L^2 \right)^{\frac{1}{2(1-\nu)}},
	\end{align}
	then:
	\begin{align}\label{rel:GMinner_linconv}
		\norm{z^+ - z(x)}^2 \le \left(1 - \frac{\alpha}{2\mu} \right) \norm{z-z(x)}^2.
	\end{align}
	By \eqref{rel:linconv_zone},  this linear decrease of residual stop when $ \norm{z-z(x)} \le \delta$, which occurs after at most
	$\left \lceil \frac{4\mu}{\alpha}\log\left( \frac{   \norm{z^0-z(x)}   }{ \delta   } \right) \right \rceil$ PsGM iterations.

	\vspace{5pt}
	
	\noindent To show the second part of our result we recall that the first assumption of \eqref{PsGM_routine_const_complexity} ensures $z(x) = \pi_{X^*}(x)$.
	By using \eqref{rel:subgrad_prelim_recurrence} with chosen subgradient $f'(x^*) = 0$, then the following recurrence is obtained:
	\begin{align}
		\norm{z^{\ell+1} \!-\! x^*}^2  
		&\!=\! \left(1 \!-\! \frac{\alpha}{\mu}\right)^2\norm{z^\ell \!-\! x^*}^2 \!-\! 2\alpha \left(1 \!-\! \frac{\alpha}{\mu}\right) \langle f'(z^\ell), z^\ell \!-\! x^*\rangle \!+\! \alpha^2\norm{f'(z^\ell)}^2 \nonumber\\
		& \le \norm{z^\ell - x^*}^2 - 2\alpha \left(1 - \frac{\alpha}{\mu}\right)\sigma_F \dist_{X^*}(z^\ell) + \alpha^2 L_f^2 \nonumber\\
		& \le \norm{z^\ell - x^*}^2 - \alpha \sigma_F \dist_{X^*}(z^\ell) + \alpha^2 L_f^2, \label{rel:reccurence_wsm}
	\end{align}
	where $x^* = \pi_{X^*}(x)$ and in the last inequality we used $\langle f'(z^t), z^t - x^*\rangle \ge F(z^t) - F^* \ge \sigma_F \dist_{X^*}(z^t) $. If $\dist_{X^*}(z^0) = \dist_{X^*}(x) > \frac{\alpha L_f^2}{2\sigma_F}$, then as long as $\dist_{X^*}(z^\ell) > \frac{\alpha L_f^2}{2\sigma_F}$, \eqref{rel:reccurence_wsm} turns into:
	\begin{align}
		\dist_{X^*}^2(z^{\ell+1}) \le	\norm{z^{\ell+1} - x^*}^2 & \le \norm{z^\ell - x^*}^2 - \frac{\left(\alpha L_f\right)^2}{2} \nonumber\\
		& \le \dist_{X^*}(x)^2 - \ell \frac{\left(\alpha L_f\right)^2}{2}.
	\end{align}
	To unify both cases, we further express the recurrence as:
	\begin{align}
		\dist_{X^*}(z^+)^2 \le \Max{ \dist_{X^*}(x)^2 - \ell\frac{\left(\alpha L_f\right)^2}{2},  \frac{\alpha^2 L_f^4}{4\sigma_F^2} },
	\end{align}
	which confirms our above result.
	
\end{proof}

\begin{remark}
	As the above theorem states, when the sequence $x^t$ is sufficiently close to the solution set, computing $x^{t+1}$ necessitates a number of PsGM iterations dependent on $\dist_{X^*}(z^0)$. In other words, the estimate from \eqref{PsGM_routine_const_complexity} can be further reduced through a good initialization or restartation technique. Such a restartation, for the neighborhood around the optimum, is exploited by the Algorithm \ref{algorithm:IPP-SGM} below.
\end{remark}


\end{document}